\documentclass[a4paper,12]{article}

\usepackage{format}
\usepackage{amsmath,amscd}
\usepackage{color}
\usepackage{graphicx}
\usepackage{epsf}
\usepackage{rsfs}
\usepackage{tikz-cd}
\usepackage{natbib}

\RequirePackage{epsfig}
\RequirePackage{amssymb}
\RequirePackage{natbib}
\RequirePackage{graphicx}

\if@usehyper
\RequirePackage[colorlinks=false,allbordercolors={1 1 1}]{hyperref}
\fi

\if@abbrvbib
\else
\fi

\bibpunct{(}{)}{;}{a}{,}{,}

\title{On the Locality of the Natural Gradient for 
Deep Learning}  

\author{Nihat Ay${}^{1,2,3}$}

\begin{document}

\maketitle

\begin{center}
${}^{1}$Max Planck Institute for Mathematics in the Sciences, Leipzig, Germany \\
${}^{2}$Leipzig University, Leipzig, Germany \\
${}^{3}$Santa Fe Institute, Santa Fe, NM, USA \\
\end{center}

\begin{center}
Email: nay@mis.mpg.de
\end{center} 
\medskip

\begin{abstract} 
We study the natural gradient method for learning in deep Bayesian networks, including neural networks. 
There are two natural geometries associated with such learning systems consisting of visible and hidden units. One geometry 
is related to the full system, the other one to the visible sub-system. These two geometries imply different natural gradients. 
In a first step, we demonstrate a great simplification of the natural gradient with respect to the first geometry, due to locality properties 
of the Fisher information matrix. This simplification 
does not directly translate to a corresponding simplification with respect to the second geometry. We develop the theory for studying 
the relation between the two versions of the natural gradient and outline a method for the simplification of the natural gradient 
with respect to the second geometry based on the first one. This method suggests to incorporate a 
recognition model as an auxiliary model for the efficient application of the natural gradient method in deep networks. \\
 
{\bf \em Keywords:\/} Natural gradient, Fisher-Rao metric, deep learning, Helmholtz machines, wake-sleep algorithm.  

{
\small
\tableofcontents
}
\end{abstract}

\section{Introduction}
\subsection{The natural gradient method}
Within the last decade, 
deep artificial neural networks have led to unexpected successes of machine learning in a large number 
of applications \citep{GBC16}. 
One important direction of research within the field 
of deep learning is based on the natural gradient method from information geometry \citep{AN00, Am16, AJLS17}. 
It has been proposed by \cite{Am98} as 
a gradient method that is invariant with respect to coordinate transformations. This method turns out to be extremely efficient within various fields 
of artificial intelligence and machine learning, including neural networks \citep{Am98}, reinforcement learning \citep{Ka02, AMR12}, and robotics \citep{PVS05}. 
It is known to overcome several problems of traditional gradient methods. Most importantly, the natural gradient method 
avoids the so-called plateau problem, and it is less sensitive to singularities in the parametrisation 
(for a detailed discussion, see Section 12.2 of \citep{Am16}; the subject of singularities is treated by \cite{Wa09}). 
On the other hand, there are significant challenges and limitations concerning the applicability of the natural gradient method \citep{Ma15}.
Without further assumptions this method becomes intractable in the context of deep neural networks which have many parameters. 
Various approximate methods have been proposed and studied as alternatives to the original method \citep{Oll15, Ku94, MG15}. 
In this article, we highlight information-geometric structures of deep Bayesian and, in particular, neural networks that allow for a simplification 
of the natural gradient. The guiding scheme of this simplification is {\em locality\/} with respect to the underlying network structure \citep{Ay02}. 
There are several aspects of learning that can be addressed from this perspective: 
\begin{enumerate}
\item {\em Objective function:\/} Typically, learning is based on the optimisation of some global objective function related to the overall performance of the network, 
which, in the most general context, is evaluated in some behaviour space. On the other hand, if we assume that 
individual units can only evaluate information accessible from their local neighbourhood, then we are naturally led to 
the following problem. Is it possible to decompose the objective function into local objective functions? 
\item {\em Learning I:\/} Assuming that learning is based on the gradient of a global objective function, 
does the above-mentioned decomposition into local functions imply a corresponding 
locality of the gradient with respect to the parametrisation? 
In that case, the individual units would adjust their parameter values, such as the synaptic connection strengths in the case of neural networks, based on local information. This is a typical implicit assumption within the field of neural networks, most prominently realsied in terms of 
Hebbian learning.   
\item {\em Learning II:\/} When computing the {\em natural\/} gradient of an objective function, we have to evaluate (the inverse of) the Fisher information matrix. 
Even if locality of learning is guaranteed for the Euclidean gradient, this matrix  
might reintroduce non-locality so that the natural gradient cannot be realised by the network in a local way. 
Therefore, we will address the following question. To what extent is the 
Fisher information matrix local? One instance of this property is that those entries of the matrix that correspond to non-local pairs of units vanish. 
This implies a block structure of the Fisher information matrix which simplifies its inversion \citep{Ay02, SN17}.      
\end{enumerate}     
\medskip

We are now going to introduce the required formalism and outline the problem setting in more detail. 

\subsection{Preliminaries and the main problem}  \label{xysetting}
We first introduce the notation used in this article. Let $\mathsf{S}$
be a non-empty finite set. We denote the canonical basis of the vector space ${\Bbb R}^{\mathsf{S}}$ by 
$e_s$, $s \in \mathsf{S}$. The corresponding dual vectors ${\delta}^s \in \left( {\Bbb R}^{\mathsf{S}} \right)^\ast$, $s \in \mathsf{S}$, defined by 
\[
    \delta^s (e_{s'}) := 
    \left\{
       \begin{array}{c@{,\quad}l}
          1 & \mbox{if $s = s'$} \\
          0 &\mbox{otherwise}
       \end{array}
    \right. ,
\]
can be identified with the Dirac measures on $\mathsf{S}$. Each linear form $l \in \left( {\Bbb R}^{\mathsf{S}} \right)^\ast$ 
can be written as $\sum_s l(s) \, \delta^s$, where 
$l(s) := l(e_s)$. We denote the open simplex of strictly positive probability vectors on $\mathsf{S}$ by
\[
  {\mathcal P} (\mathsf{S}) \, := \, \left\{ p = \sum_s p(s) \, \delta^s \; : \; \mbox{$p(s) > 0$ for all $s$, and $\displaystyle \sum_s p(s) = 1$} \right\}.  
\]
For each point $p \in {\mathcal P}({\mathsf{S}})$, the tangent space in $p$  
can be naturally identified with   
\[
     {\mathcal T}(\mathsf{S}) \, :=  \, 
     \left\{ V = \sum_s V(s) \, \delta^s \; : \; \sum_{s} V(s) = 0 \right\}.
\]
The Fisher-Rao metric on ${\mathcal P}(\mathsf{S})$ in $p = \sum_s p(s) \, \delta^s$ is defined by
\begin{equation} \label{fisherrao1}
     {\langle V, W \rangle}_p := \sum_{s} \frac{1}{p(s)} \, V(s) W(s), \qquad V,W \in \mathcal{T}({\mathsf{S}}). 
\end{equation}
Let us now consider a model ${\mathcal M} \subseteq {\mathcal P}(\mathsf{S})$ which we assume to be a
$d$-dimensional smooth manifold with local coordinates $\xi = (\xi_1,\dots,\xi_d) \mapsto p_\xi$, where $\xi$ is from an open domain $\Xi$ in ${\Bbb R}^d$. 
Below, we will treat more general models, 
but starting with manifolds allows us to outline more clearly the challenges we face in the context of the natural gradient method.    
With $p(s ; \xi) := p_\xi(s)$, we obtain from (\ref{fisherrao1}) the 
Fisher information matrix $G(\xi) = {\left( g_{ij} (\xi) \right)}_{ij}$ defined by 
\begin{equation} \label{fisherinf}
     g_{i j} (\xi) \, := \, \left\langle \partial_{i} p_\xi , \partial_{j} p_\xi \right\rangle_{p_\xi}  \, = \, 
                   \sum_{s} p(s ; \xi) \, \frac{\partial \ln p(s ; \cdot)}{\partial \xi_i} (\xi) \frac{\partial \ln p(s ; \cdot)}{\partial \xi_j} (\xi).  
\end{equation}
\medskip

In this article, the set $\mathsf{S}$ will typically be a Cartesian product of state sets of units, for instance binary neurons.
More precisely, we consider a non-empty and finite set $N$ of units consisting of $n$ visible units $V$ and $m$ hidden units $H$, that is $N = V \uplus H$. 
The state sets of the units are denoted by $\mathsf{X}_i$, $i \in N$, and assumed to be non-empty and finite. For any subset $A \subseteq N$, 
we have the corresponding configuration or state set $\mathsf{X}_A := \times_{i \in A} \mathsf{X}_i$, the set ${\mathcal P}_A := {\mathcal P}(\mathsf{X}_A)$ 
of strictly positive probability vectors on $\mathsf{X}_A$, and the tangent space ${\mathcal T}_A := {\mathcal T}(\mathsf{X}_A)$.     
Consider now the restriction $X_V: \mathsf{X}_V \times \mathsf{X}_H \to \mathsf{X}_V$, $(v , h) \mapsto v$, and its push-forward map 
\[
    \pi_V: \; {\mathcal P}_{V,H} := {\mathcal P}_N \; \to \; {\mathcal P}_V, \qquad p \; \mapsto \; \pi_V (p) \, := \, \sum_{v \in \mathsf{X}_V} p(v) \, \delta^{v},
\]  
where $p (v) := \sum_{h \in \mathsf{X}_H}  p(v, h)$. This is simply the marginalisation map where $\pi_V(p)$ is the $V$-marginal of $p$. Given a model 
${\mathcal M}$ in ${\mathcal P}_{V,H}$, we consider the projected model ${\mathcal M}_V := \pi_V ({\mathcal M})$ in ${\mathcal P}_V$ which will play a 
major role in this article. Before we come to this, let us first observe a number of challenges that appear already at this point.   
\begin{enumerate}
    \item Even if we choose ${\mathcal M}$ to be a smooth manifold, its projection ${\mathcal M}_V$ is typically a much more 
    complicated geometric object with various kinds of singularities. Throughout this article, we will allow for more general models without assuming 
    ${\mathcal M}$ to be a smooth manifold in the first place. However, we will restrict attention to non-singular points only.  
    \item Having a general model ${\mathcal M}$, we also drop the assumption that the 
    parametrisation $\xi = (\xi_1,\dots,\xi_d) \mapsto p_\xi$ is given by a (diffeomorphic) coordinate system. 
    This has consequences on the definition of the Fisher-Rao 
    metric in a non-singular point $p_\xi$:
    \begin{enumerate} 
       \item In order to interpret the Fisher-Rao metric as a Riemannian metric, the derivatives $\frac{\partial}{\partial \xi_i} p_\xi$, $i = 1,\dots,d$, 
       have to span the whole tangent space $T_\xi {\mathcal M}$ in $p_\xi$. (This is often implicitly assumed but not 
       explicitly stated.) Otherwise, the Fisher-Rao metric defined by (\ref{fisherinf}) will 
       not be positive definite. We will refer to parametrisations that satisfy this condition as proper parametrisations.  Note that for a proper parametrisation $\xi \mapsto p_\xi$ of ${\mathcal M}$, the composition $\xi \mapsto \pi_V(p_\xi)$ 
       is not necessarily a proper parametrisation of ${\mathcal M}_V$.           
       \item Another consequence of not having a coordinate system as a parametrisation is the fact that the number $d$ of parameters may exceed the dimension 
       of the model. Even if we assume ${\mathcal M}$ to be a smooth manifold and its parametrisation given by a coordinate system, 
       such that $d$ equals the dimension of ${\mathcal M}$,
       the corresponding projected model ${\mathcal M}_V$ can have a much lower dimension. 
       In that case, we say that the model is overparametrised. Such models play an important role within the field of deep learning. 
       The Fisher-Rao metric for such models is well defined in non-singular points. However the Fisher information matrix (\ref{fisherinf}) will be 
       degenerate so that the representation of a gradient in terms of the parameters is not unique anymore. Below, we will come back to this problem.        
    \end{enumerate}
\end{enumerate}
\medskip

We use the natural gradient method in order to minimise (or maximise) 
a function $f: {\mathcal M}_V \to {\Bbb R}$ 
which is usually obtained as a restriction of a smooth function defined on ${\mathcal P}_V$. 
Therefore, it is natural to use the Fisher-Rao metric on ${\mathcal M}_V$ inherited from ${\mathcal P}_V$. Assuming that 
all required quantities are well defined, we can express this natural gradient in terms of the parametrisation as
\begin{equation} \label{gradinloc}
   {\rm grad}_{\xi} f \; = \; G^{+}(\xi) \nabla_\xi f ,
\end{equation}
where $G^+(\xi)$ is the Moore-Penrose inverse of the Fisher information matrix $G(\xi) = {(g_{ij})}_{ij}$ defined by (\ref{fisherinf}). If the parametrisation 
is given by a coordinate system then this reduces to the ordinary matrix inverse (see the Appendix for more details on the Moore-Penrose inverse).  
The general difficulty that we face with equation (\ref{gradinloc}) is the inversion of the Fisher information matrix, especially in deep networks with 
many parameters. On the other hand, the model ${\mathcal M}_V$ is obtained as the image of the model   
${\mathcal M}$ which can be easier to handle, despite the fact that it ``lives'' in the larger space ${\mathcal P}_{V,H}$.  
Instead of optimising the function $f$ on ${\mathcal M}_V$ we can try to optimise the pull-back of $f$, $\widetilde{f} := f \circ \pi_V$, 
defined on ${\mathcal M}$. But this creates a conceptual problem related to the very nature of the natural gradient method. 
As ${\mathcal M}$ 
inherits the Fisher-Rao metric from ${\mathcal P}_{V,H}$, we can express the corresponding gradient as 
\begin{equation} \label{gradinloc2}
   {\rm grad}_{\xi} \widetilde{f} \; = \; \widetilde{G}^{+}(\xi) \nabla_\xi \widetilde{f},
\end{equation}
where $\widetilde{G}(\xi)$ denotes the Fisher information matrix in $p_\xi \in {\mathcal M}$. 
This can simplify the problem in various ways. As already outlined, ${\mathcal M}_V$ typically has singularities, even if ${\mathcal M}$ is a smooth manifold. 
In that case, the gradient (\ref{gradinloc2}) is well defined for all $\xi$, whereas the gradient (\ref{gradinloc}) is not. 
A further simplification comes from the fact that ${\mathcal M}$ is typically associated with some network, 
which implies a block structure of the Fisher information matrix on ${\mathcal M}$. In Section \ref{bnn}, 
we will demonstrate this simplification for models that are associated with directed acyclic graphs, where the elements of ${\mathcal M}$
factorise accordingly. With this simplification, the inversion of $\widetilde{G}(\xi)$ can become much easier than the inversion of $G(\xi)$ (when the latter is defined). 
On the other hand, if we consider the model ${\mathcal M}_V$ to be the prime model, 
where the hidden units play the role of auxiliary units, then we {\em have\/} to use the information geometry of ${\mathcal M}_V$ for learning. 
Therefore, it is important to relate the corresponding natural gradients, that is (\ref{gradinloc}) and (\ref{gradinloc2}), to each other. 
This is done in a second step, presented in Section \ref{singularnatgrad}. In particular, we will identify conditions for the equivalence of the two gradients, leading 
to a new interpretation of Chentsov's classical characterisation of the Fisher-Rao metric in terms of its invariance with respect to Markov morphisms 
\citep{Che82}. (A general version of this characterisation is provided by \cite{AJLS17}.)  Based on the comparison of the gradients (\ref{gradinloc}) and (\ref{gradinloc2}), we will analyse how to extend 
locality properties of learning that hold for ${\mathcal M}$ to the model ${\mathcal M}_V$. 
This analysis is closely related to the above-mentioned approximate methods as alternatives to the natural gradient method.  
Of particular relevance in this context is the replacement of the Fisher information matrix by the unitwise Fisher information matrices 
as studied in \citep{Oll15, Ku94}. Note, however, 
that we are not aiming at {\em approximating\/} the natural gradient on ${\mathcal M}_V$ by the unitwise natural gradient. 
In this article, we aim at identifying conditions for their {\em equivalence\/}.     
Furthermore, in order to satisfy these conditions we propose an extension $\widetilde{\mathcal M}$ of ${\mathcal M}$ 
which corresponds to an interesting extension of the underlying network.   
This will lead us to a new interpretation of so-called recognition models, which are used in the context of Helmholtz machines 
and the wake-sleep algorithm \citep{DHNZ95, HDFN95, ND97}. 
Information-geometric works on the wake-sleep algorithm and its close relation to the $em$-algorithm 
are classical \citep{Am95, FA95, IAN98}. More recent contributions to the information geometry of the wake-sleep algorithm are provided by    
\cite{BSFB16} and \cite{VVMA20}. Directions of related research in view of this article are outlined in the conclusions, Section \ref{conclu}. 

\section{Locality of deep learning in Bayesian and neural networks} \label{bnn}
\subsection{Locality of the Euclidean gradient} \label{loceucl}
We now define a sub-manifold of ${\mathcal P}_{V,H}$ in terms of a directed acyclic graph
$G = (N, E)$ where $E$ is the set of directed edges.  
With each node $r$ we associate a local Markov kernel 
$k^r$, that is a map $\mathsf{X}_{pa(r)} \times \mathsf{X}_r \to [0,1]$, $(x_{pa(r)}, x_r) \mapsto k^r(x_r | x_{pa(r)})$, which satisfies 
$\sum_{x_r} k^r(x_r | x_{pa(r)}) = 1$.  
Given such a family of Markov kernels, we define the joint distribution 
\begin{equation} \label{prod} 
     p(x_N) \, = \, \prod_{r \in N} k^r(x_r | x_{pa(r)}).  
\end{equation}
Strictly positive distributions of the product structure (\ref{prod}) form a manifold ${\mathcal M}$. 
A natural subset of ${\mathcal M}$ is given by the product distributions, that is those distributions of the form
\[
     p(x_N) \, = \, \prod_{r \in N} p(x_r).
\]
In order to treat manifolds ${\mathcal M}$ given by a neural network, a so-called {\em neuromanifold\/}, 
we consider more general parametrisations $\xi_{r} = (\xi_{(r;1)}, \dots , \xi_{(r ; d_r)}) \mapsto \kappa^{r}_{\xi_r}$
(we also use the notation $k^r(x_r  | x_{pa(r)} ; \xi_r)$ for $k^r_{\xi_r}(x_r | x_{pa(r)})$).
This defines ${\mathcal M}$ as the image of the map
\begin{equation} \label{parametris} 
    \xi \; \mapsto \; p_\xi(x_V, x_H) \, = \, \prod_{r \in N}  k^r(x_r | x_{pa(r)}; \xi_r ) .
\end{equation}
In order to use matrix notation, we sometimes assume, without loss of generality, $N = \{1,2,\dots, n + m\}$ such that $r \leq s$ whenever $r \in pa(s)$.
\medskip

Now we come to the main objective of learning as studied in this article. 
In many applications, one tries to represent a target probability vector $p^\ast \in {\mathcal P}_V$ on the state set of visible units (or a target 
conditional probability vector). Assume that we have a sub-manifold ${\mathcal M}$ of ${\mathcal P}_{V,H}$, and consider the image  
${\mathcal M}_V := \pi_V ({\mathcal M}) \subseteq {\mathcal P}_V$. Given a target probability vector $p^\ast \in {\mathcal P}_V$, the task is to find the best 
approximation of $p^\ast$ by members of ${\mathcal M}_V$:
\begin{equation} \label{minimisation}
       D(p^\ast \| {\mathcal M}_V ) := \inf_{q \in {\mathcal M}_V} D(p^\ast  \|  q).
\end{equation}
Here, $D(p \|  q )$ denotes the Kullback-Leibler divergence
\[
    D(p \| q) = \sum_{x} p(x) \ln \frac{p(x)}{q(x)}
\]
between $p$ and $q$. With the parametrisation (\ref{parametris}) of the elements of ${\mathcal M}$, we consider the function 
\begin{equation}
     E(\xi) \, := \, D(p^\ast \, \| \, q_V(\xi)) \, = \, \sum_{x_V} p^\ast(x_V) \ln \frac{p(x_V)}{p(x_V ; \xi)} \, = \, 
     \sum_{x_V} p^\ast(x_V) \ln \frac{p^\ast(x_V)}{\sum_{x_H} p(x_V, x_H; \xi)}. \label{oriopi}
\end{equation}
Minimisation of $E$ can be realised in terms of the gradient method. In this section we begin with the Euclidean gradient which is determined by the 
partial derivatives of $E$. It is remarkable that, even though the network can be large, with many hidden units, the resulting derivatives are local in a 
very useful way (see a similar derivation in the context of sigmoid belief networks by \cite{Ne92}):
\begin{eqnarray}
   \frac{\partial E}{\partial  \xi_{(r;i)}} (\xi) 
      & = & \frac{\partial}{\partial \xi_{(r;i)}} \sum_{x_V} p^\ast(x_V) \ln \frac{p^\ast(x_V)}{p(x_V; \xi)}  \nonumber \\
      & = &  - \sum_{x_V} p^\ast(x_V) \frac{\partial \ln p(x_V; \cdot)}{\partial \xi_{(r;i)}} (\xi)  \nonumber \\
      & = & - \sum_{x_V} \frac{p^\ast(x_V)}{p_\xi(x_V)}  \sum_{x_H} \frac{\partial p(x_V, x_H; \cdot)}{\partial \xi_{(r;i)}} (\xi)  \nonumber  \\
      & = & - \sum_{x_V} \frac{p^\ast(x_V)}{p_\xi(x_V)}  \sum_{x_H} \frac{\partial  }{\partial \xi_{(r;i)}} \prod_{s \in N}  k^s(x_s | x_{pa(s)} ; \xi_s)  \nonumber \\
      & = & - \sum_{x_V} \frac{p^\ast(x_V)}{p_\xi(x_V)}  \sum_{x_H} \prod_{s \in N \atop s \not= r}  k^s (x_i | x_{pa(s)} ; \xi_s)  
           \frac{\partial  k^r (x_r | x_{pa(r)} ; \cdot)}{\partial \xi_{(r;i)}} (\xi_r)  \nonumber \\
      & = & - \sum_{x_V} \frac{p^\ast(x_V)}{p_\xi(x_V)}  \sum_{x_H} 
                \prod_{s \in N}  k^s (x_s | x_{pa(s)} ; \xi_s)  \frac{\partial  \ln k^r (x_r | x_{pa(r)} ; \cdot) }{\partial \xi_{(r;i)}} (\xi_r)  \nonumber \\  
      & = & - \sum_{x_V,x_H} \frac{p^\ast(x_V)}{p_\xi(x_V)}  p_\xi(x_V, x_H)
                 \frac{\partial  \ln k^r(x_r | x_{pa(r)} ; \cdot) }{\partial \xi_{(r;i)}} (\xi_r)  \nonumber \\  
      & = & - \sum_{x_V,x_H} p^\ast(x_V) \,  p_\xi(x_H \, | \, x_V) \,
                 \frac{\partial}{\partial \xi_{(r;i)}} \ln k^r(x_r | x_{pa(r)} ; {\xi_r} )  \nonumber  
\end{eqnarray}
With $p^\ast(x_V,x_H ; \xi) := p^\ast(x_V) \, p(x_H \, | \, x_V ; \xi)$, we finally obtain
\begin{eqnarray}                 
   \frac{\partial E}{\partial  \xi_{(r;i)}} (\xi)  & = & - \sum_{x_V,x_H} p^\ast(x_V , x_H ; \xi) \,
                 \frac{\partial}{\partial \xi_{(r;i)}} \ln k^r (x_r | x_{pa(r)} ; {\xi_r})  \nonumber \\   
      & = & - \sum_{x_{pa(r)}} p^\ast(x_{pa(r)} ; \xi ) \sum_{x_r} p^\ast(x_r | x_{pa(r)} ; \xi) \,
                 \frac{\partial}{\partial \xi_{(r;i)}} \ln  k^r (x_r | x_{pa(r)} ; {\xi_r}). \label{derivative}                                       
\end{eqnarray}
We have an expectation value of a function, $\ln  k^r (x_r | x_{pa(r)} ; {\xi_r})$, that is local in two ways: all arguments of this function, 
the states {\em and\/} the parameters, are local with respect to the node $r$. 
However, the distribution $p^\ast_\xi$, used for the evaluation of the expectation value, 
depends on the full set of parameters $\xi$. On the other hand, due to the locality of $\ln  k^r (x_r | x_{pa(r)} ; {\xi_r})$ with respect to the states $x_{pa(r)}$ and $x_r$, this expectation value depends only on the marginal $p^\ast(x_{pa(r)}, x_r)$.
One natural way to approximate (\ref{derivative}) is by sampling from this distribution. 
This is typically difficult, 
compared to the sampling from $p_\xi$ which factorises according to the underlying directed acyclic graph $G$. ``One-shot sampling'' from $p_\xi$ 
is possible by simply using $p_\xi$ as a generative model (recursive application of the local kernels $k^r_{\xi_r}$ according to the underlying directed acyclic graph). As $p^\ast_\xi$  incorporates the target distribution $p^\ast$ on $\mathsf{X}_V$ and does not necessarily factorise according to $G$, 
sampling from it has to run much longer. For completeness, the Gibbs sampling method is outlined in more detail at the end of this section. 

\subsection{The wake-sleep algorithm} \label{wsa}
We now highlight an important alternative to sampling from $p^\ast_\xi$ for the computation of the derivative (\ref{derivative}). This alternative is based on the idea 
that we have, in addition to the generative model ${\mathcal M}$ of distributions $p_\xi$, a so-called {\em recognition model\/} ${\mathcal L}_{H|V}$ 
of conditional distributions $q(x_H | x_V ; {\eta})$ with which we can approximate 
$p(x_H \, | \, x_V ; \xi)$. As a consequence, such a recognition model allows us to approximate (\ref{derivative}) where we replace 
$p^\ast(x_V, x_H ; \xi) = p^\ast(x_V) \, p(x_H | x_V ; \xi)$ by $q^\ast(x_V, x_H ; \eta) :=  p^\ast(x_V) \, q(x_H | x_V ; \eta)$, 
and correspondingly the marginals on $pa(r) \cup \{r\}$. We obtain    
\begin{eqnarray}                 
   \frac{\partial E}{\partial  \xi_{(r;i)}} (\xi)  
      & \approx & - \sum_{x_{pa(r)}} q^\ast(x_{pa(r)} ; \eta) \sum_{x_r} q^\ast(x_r | x_{pa(r)} ; \eta) \,
                 \frac{\partial}{\partial \xi_{(r;i)}} \ln  k^r(x_r | x_{pa(r)} ; {\xi_r}) \label{derivative1}       \label{approxi1} \\
      & = &    - \frac{\partial}{\partial \xi_{(r;i)}} \sum_{x_{pa(r)}} q^\ast(x_{pa(r)} ; \eta ) \sum_{x_r} 
                   q^\ast(x_r | x_{pa(r)} ; \eta )  \ln k^r(x_r | x_{pa(r)} ; {\xi_r}) \label{approxi2} \\              
      & = &    \frac{\partial}{\partial \xi_{(r;i)}} 
                  \sum_{x_{pa(r)}} q^\ast(x_{pa(r)} ; \eta ) \sum_{x_r} q^\ast(x_r | x_{pa(r)} ; \eta ) 
                  \ln \frac{q^\ast(x_r | x_{pa(r)} ; \eta )}{k^r(x_r | x_{pa(r)} ; {\xi_r})} . \label{approxi3}                                
\end{eqnarray}       
For the evaluation of the gradient of $E$ with respect to the $\xi$-parameters we can now sample with the recognition model, instead of the generative model. 
This approximation will be the more accurate the smaller the following relative entropy is: 
\begin{equation} \label{diffgenrec}
     D(\xi \| \eta ) \; := \; \sum_{x_V}  p(x_V ; \xi) \sum_{x_H} p(x_H | x_V ; \xi) \ln \frac{p(x_H | x_V ; \xi )}{q(x_H | x_V ; \eta )} .
\end{equation}
Ideally, we would like the recognition model to be rich enough to represent the conditional distributions of the generative model. 
More precisely, we assume that for all $\xi$, there is an $\eta = \eta(\xi)$ so that $q(x_H | x_V ; \eta ) = p(x_H | x_V ; \xi )$. Furthermore, for (\ref{diffgenrec}) 
to be tractable, we assume that $q(x_H | x_V ; \eta )$ also factorises according to some directed acyclic graph $G'$, so that 
\begin{equation} \label{facrec}
   q(x_H | x_V ; \eta) 
      \; = \; \prod_{r \in H} l^r(x_r | x_{pa'(r)} ; \eta_r) ,
\end{equation}
where ${pa}'(r)$ denotes the parent set of the node $r$ with respect to the graph $G'$. With these assumptions, the expressions 
(\ref{diffgenrec}) simplifies considerably, and we obtain 
\begin{eqnarray}
\frac{\partial D(\xi \| \cdot)}{\partial  \eta_{(r;j)}} (\eta) 
   & = & - \frac{\partial}{\partial  \eta_{(r;j)}}  \sum_{x_{pa'(r)}} p(x_{pa'(r)} ; \xi) \sum_{x_r} p(x_r | x_{pa'(r)} ; \xi) \ln l^r(x_r | x_{pa'(r)} ; \eta_r) \label{approxr1} \\
   & = & \frac{\partial}{\partial  \eta_{(r;j)}}  \sum_{x_{pa'(r)}} p(x_{pa'(r)} ; \xi) \sum_{x_r} p(x_r | x_{pa'(r)} ; \xi) \label{approxr2}
             \ln \frac{p(x_r | x_{pa'(r)} ; \xi)}{l^r(x_r | x_{pa'(r)} ; \eta_r)}.
\end{eqnarray}
Note that, while $p_\xi$ factorises according to $G$ so that the conditional distribution $p(x_r | x_{pa(r)} ; \xi)$ coincides with the kernel 
$k^r( x_r | x_{pa'(r)} ; \xi)$, the conditional distribution $p(x_r | x_{pa'(r)} ; \xi)$ with respect to $G'$ does not have a correspondingly simple structure. 
On the other hand, we can easily sample from $p_\xi$, and thereby also from $p(x_{pa'(r)} ; \xi)$ and $ p(x_r ; \xi)$, using the product structure with respect to $G$.    
\medskip

Let us now come back to the original problem of minimising $E$ with respect to $\xi$ based on the gradient descent method. If the parameter $\eta$ of the 
recognition model is such that $q(x_H | x_V ; \eta) = p(x_H | x_V ; \xi)$  
then the approximation (\ref{approxi1}) is exact, and we can evaluate the partial derivatives $\partial / \partial \xi_{(r ; i)}$ 
by sampling from $q^\ast(x_V, x_H; \eta) = p^\ast(x_V) \, q(x_H | x_V; \eta)$. 
This can then be used for updating the parameter $\xi$, say from 
$\xi$ to $\xi + \Delta \xi$ where $\Delta \xi$ is proportional to the euclidean gradient. As this update is based on sampling from the target distribution 
$p^\ast(x_V)$ and the recognition model $q(x_H | x_V ; \eta)$, it is referred to as the {\em wake phase\/}.   
After this update, we typically have 
$q(x_H | x_V ; \eta) \not= p(x_H | x_V ; \xi + \Delta \xi)$.
In order to use (\ref{approxi1}) for the next update of $\xi$, we therefore have to readjust $\eta$, say from $\eta$ to $\eta + \Delta \eta$, so that we recover the identity 
$q(x_H | x_V ; \eta + \Delta \eta) = p(x_H | x_V ; \xi + \Delta \xi)$. This can be achieved by choosing $\Delta \eta$ to be proportional to the 
euclidean gradient (\ref{approxr1}) with respect to $\eta$. The evaluation of the partial derivatives $\partial / \partial \eta_{(r ; j)}$ requires sampling from the 
generative model $p(x_V, x_H ; \xi)$, with no involvement of the target distribution $p^\ast(x_V)$. This is the reason why the $\eta$-update is referred to as the 
{\em sleep phase\/}. Alternating application of the wake phase and the sleep phase yields the so-called {\em wake-sleep algorithm\/}, which has been introduced and 
studied in the context of neural networks by \cite{DHNZ95, HDFN95, ND97}.
It has been pointed out that this algorithm cannot be interpreted as a gradient decent algorithm of a potential function on both variables $\xi$ and $\eta$. 
On the other hand, here we derived the wake-sleep algorithm as a gradient decent algorithm for the optimisation of the objective function $E$ which only depends on the variable $\xi$. The auxiliary variable $\eta$ is used for the approximation of the gradient of $E$ with respect to $\xi$. In order to have a good approximation of this 
gradient, we have to apply the sleep phase update more often, until convergence of $\eta$. Only then, we can update 
$\xi$ within the next wake phase. With this asymmetry of time-scale for the two phases, the wake-sleep algorithm {\em is\/} a gradient decent 
algorithm for $\xi$, which has been outlined in the context of the $em$-algorithm by \cite{IAN98}. 
\medskip

We have introduced the parameters $\eta$ for sampling and thereby evaluating the derivative (\ref{derivative}). However, there is another remarkable feature of the corresponding extended optimisation problem. While the original optimisation function $E$, defined by (\ref{oriopi}), does not appear to be local in any sense, the  
extended optimisation in terms of a generalised wake-sleep algorithm, which is equivalent to the original problem, is based on a set of local functions associated with the respective units. More precisely, the expressions (\ref{approxi2}) and (\ref{approxr1}) are derivatives of local cross entropies, whereas the expressions 
(\ref{approxi3}) and (\ref{approxr2}) are derivatives of local KL-divergences.  
\medskip

We conclude with the important note that a recognition model which, on the one hand, is rich enough to represent all distributions $p(x_H | x_H ; \xi)$ and, on the other hand, factorises according to (\ref{facrec}) might require a large graph $G'$ and a correspondingly large number of parameters $\eta_{(r; j)}$ which constitute the vector $\eta$. In practice, the recognition model is typically chosen to be of the same dimensionality as the generation model and does not necessarily satisfy the above conditions.     

\subsection{Gibbs sampling} 
By holding the configuration $x_V$ constant, we can sample from $p(x_H | x_V; \xi)$ by randomly selecting a node $s \in H$, and then updating 
the state of that node according to $p( x_s | x_{H \setminus s}, x_V ; \xi)$. After this update we repeat choosing a node and updating its state. 
This will generate, after many repetitions, $p_\xi^\ast$-typical patterns.  
The conditional distribution is simple because, due to the local Markov property, it satisfies 
$p( x_s | x_{H \setminus s}, x_V; \xi) = p( x_s | x_{bl(s)} ; \xi)$, where $bl(s)$ denotes the Markov blanket of $s$ (see Figure \ref{fig:markovblanket}).  
It is defined as 
\[
    bl(s) := pa(s) \cup ch(s) \cup \bigcup_{j \in ch(s)} ( pa(j) \setminus s ).
\]  
Note that, in general, the Markov blanket is larger than the parent set and $p( x_s | x_{bl(s)} ; \xi)$ differs from 
$k^s ( x_s | x_{pa(s)} ; {\xi_s})$. 
More precisely, for $s \in H$, we have 
\begin{eqnarray} \nonumber
   p(x_s | x_{bl(s)} ; \xi) 
        & = & p(x_s | x_V, x_{H \setminus s} ; \xi) \nonumber \\
        & = & \frac{ p(x_V, x_{H \setminus s},x_s ; \xi)}{\sum_{x_s'} p(x_V, x_{H \setminus s},x_s'; \xi)}  \nonumber \\ 
        & = & \frac{ \prod_{i \in N} k^i (x_i | x_{pa(i)} ; {\xi_i})}{\sum_{x_s'} 
                  \prod_{i \in N} k^i (x_i' | x_{pa(i)} ; {\xi_i} )}  \nonumber \\
        & = &  \frac{ k^s(x_s | x_{pa(s)} ; {\xi_s}) \prod_{i \in ch(s)} k^i (x_i | x_{pa(i)} ; {\xi_i})}{
                  \sum_{x_s'} k^s(x_s' | x_{pa(s)} ; {\xi_s}) \prod_{i \in ch(s)} k^i (x_i | x_{pa(i) \setminus s}, x_s' ; {\xi_i})}.    \label{localcomp}                   
\end{eqnarray}
Even though we cannot use the local kernels $k^s ( x_s | x_{pa(s)} ; {\xi_s})$ as a generative model for sampling from $p^\ast$, the kernels 
$p(x_s | x_{bl(s)} ; \xi)$, which are used for Gibbs sampling, are still local in the sense that they only depend on the Markov blanket of $s$ 
(see Figure \ref{fig:markovblanket}).  
\begin{figure}[h]        
  \centering
       \includegraphics[width=6cm]{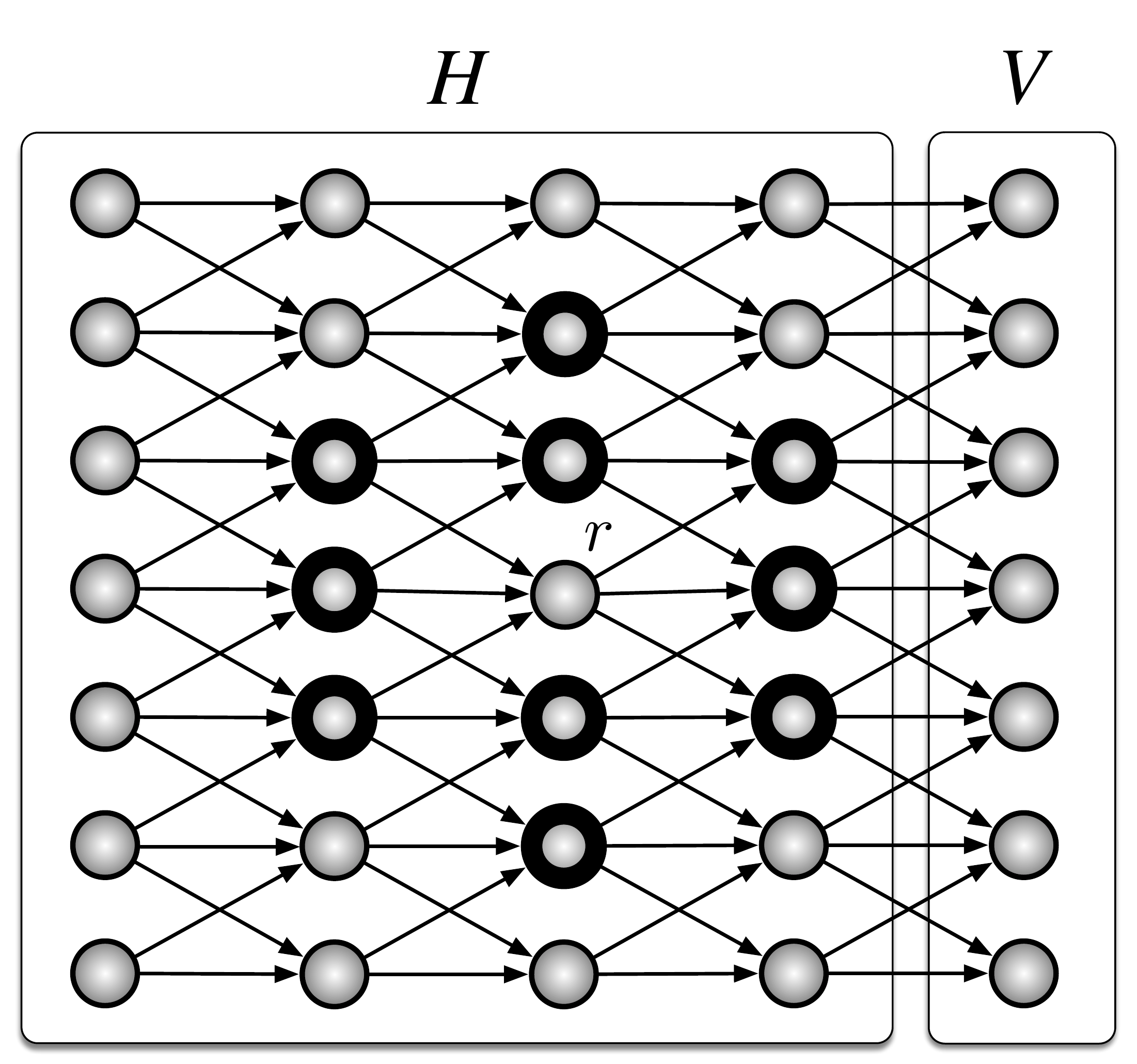}   
  \caption{Illustration of the Markov blanket $bl(r)$ of a node $r$. Its elements are highlighted by thick circles.} 
  \label{fig:markovblanket}
\end{figure} 

We exemplify the derivations of this section in the context of binary neural networks. 

\begin{example}[Neural networks (I)] \label{nn1}
We assume that the units $r \in N$, referred to as {\em neurons\/} in this context, are binary with state sets $\{-1,+1\}$.       
For each neuron $r$, we consider a vector $w_r = (w_{ir})_{i \in pa(r)}$
of {\em synaptic connection strengths\/} and a {\em threshold value\/} $\vartheta_r$. (For a synaptic strength $w_{ir}$, 
$i$ is referred to as the {\em pre-synaptic\/} and $r$ the {\em post-synaptic neuron\/}, respectively.)
We set $\xi_{(r;i)} := w_{ir}$, $i = 1,\dots, d_r - 1$, and $\xi_{(r;d_r)} := \vartheta_r$, that is $\xi_r = (w_r, \vartheta_r)$.   
In order to update its state, the neuron first evaluates the local function 
\[
      h_r(x_{pa(r)}) := \sum_{i \in pa(r)} w_{i r} x_i  - \vartheta_r  
\]
and then generates a state $x_r\in \{-1,+1\}$ with probability
\begin{eqnarray}
       k^r(x_r | x_{pa(r)}; w_r , \vartheta_r)  
            & := &   \frac{1}{1 + e^{- x_r h_r(x_{pa(r)})}}. \label{updateprob1}         
\end{eqnarray}
We calculate the derivatives
\begin{eqnarray}
   \frac{\partial}{\partial w_{i r}}  \ln k^r(x_r | x_{pa(r)}; w_r , \vartheta_r) 
             & = &  \frac{x_i x_r}{1 + e^{x_r h_r(x_{pa(r)})}} , \label{derivativesyn}  \\
  \frac{\partial}{\partial \vartheta_{r}}  \ln k^r(x_r | x_{pa(r)}; w_r , \vartheta_r) 
             & = & - \frac{x_r}{1 + e^{x_r h_r(x_{pa(r)})}} , \label{derivativethr} 
\end{eqnarray}
and, with (\ref{derivative}), we obtain  
\begin{eqnarray}
   \frac{\partial E}{\partial w_{ir}}  (w, \vartheta)
        & = & - \sum_{x_{pa(r)}} \sum_{x_r} p^\ast_\xi(x_{pa(r)} , x_r) \, \frac{x_i x_r}{1 + e^{x_r h_r(x_{pa(r)})}}, \label{hebb} \\
   \frac{\partial E}{\partial \vartheta_r}  (w, \vartheta)
      & = &  \sum_{x_{pa(r)}} \sum_{x_r} p^\ast_\xi(x_{pa(r)} , x_r) \,\frac{x_r}{1 + e^{x_r h_r(x_{pa(r)})}}. \label{single}
\end{eqnarray}
Equation (\ref{hebb}) is one instance of the {\em Hebb rule\/} which is based on the learning paradigm phrased as ``cells that fire together wire together'' \cite{He49}. 
Note, however, that the causal interpretation of the underlying directed acyclic graph ensures that 
the pre-synaptic activity $x_i$ is measured {\em before\/} the post-synaptic activity $x_r$. This causally consistent version of the Hebb rule has been 
experimentally studied in the context of {\em spike-timing-dependent plasticity\/} of real neurons (e.g., \cite{BP01}).      

In order to evaluate the derivatives (\ref{hebb}) and (\ref{single}), we have to sample from $p^\ast_\xi$. We use 
Gibbs sampling based on the expression (\ref{localcomp}) which, for binary state sets $\{-1,+1\}$, reduces to
\begin{eqnarray*}
   p(x_s | x_V, x_{H \setminus s} ; \xi) & = & \frac{1}{1 + 
       \frac{k^s (- x_s | x_{pa(s)} ; {\xi_s} )}{k^s(x_s | x_{pa(s)} ; {\xi_s})}  {\prod}_{i \in ch(s)}
       \frac{k^i(x_i | x_{pa(i)\setminus s},-x_s ; {\xi_i})}{k^i (x_i | x_{pa(i)\setminus s}, x_s ; {\xi_i})}}.
\end{eqnarray*}
Let us analyse the term in the denominator. Using the update rule (\ref{updateprob1}), simple calculations yield 
\begin{eqnarray*}
\lefteqn{\frac{k^s(- x_s | x_{pa(s)} ; {\xi_s})}{k^s (x_s | x_{pa(s)} ; {\xi_s})}  {\prod}_{i \in {ch(s)}}
       \frac{k^i (x_i | x_{pa(i)\setminus s},-x_s ; {\xi_i})}{k^i (x_i | x_{pa(i)\setminus s}, x_s ; {\xi_i})}} \\
       & = & \frac{e^{-\frac{1}{2} x_s h_s(x_{pa(s)})}}{e^{\frac{1}{2} x_s h_s(x_{pa(s)})}} \times \\
       &  & \qquad 
                 \prod_{i \in {ch(s)}} 
                 \frac{e^{\frac{1}{2} x_i \left( h_i(x_{pa(i)}) - 2 w_{si} x_s \right)}}{
                 e^{\frac{1}{2} x_i \left( h_i(x_{pa(i)}) - 2 w_{si} x_s \right)} + e^{-\frac{1}{2} x_i \left(h_i(x_{pa(i)}) - 2 w_{si} x_s \right)}}
                 \frac{e^{\frac{1}{2} x_i h_i(x_{pa(i)})} + e^{-\frac{1}{2} x_i h_i(x_{pa(i)})}}{e^{\frac{1}{2} x_i h_i(x_{pa(i)})}} \\
        & = &  e^{- x_s h_s(x_{pa(s)})}
                 \prod_{i \in {ch(s)}} 
                 \frac{e^{- w_{si} x_s x_i } \left( e^{\frac{1}{2} x_i h_i(x_{pa(i)})} + e^{-\frac{1}{2} x_i h_i(x_{pa(i)})} \right)}{
                 e^{- w_{si} x_s x_i } e^{\frac{1}{2} x_i h_i(x_{pa(i)})} + e^{w_{si} x_s x_i } e^{-\frac{1}{2} x_i h_i(x_{pa(i)})}} \\  
        & = &  e^{- x_s h_s(x_{pa(s)})}
                 \prod_{i \in {ch(s)}} 
                 \frac{e^{\frac{1}{2} x_i h_i(x_{pa(i)})} + e^{-\frac{1}{2} x_i h_i(x_{pa(i)})}}{
                 e^{\frac{1}{2} x_i h_i(x_{pa(i)})} + e^{2 w_{si} x_s x_i } e^{-\frac{1}{2} x_i h_i(x_{pa(i)})}} \\     
       & = &  e^{- x_s h_s(x_{pa(s)})}
                 \underbrace{\prod_{i \in {ch(s)}} 
                 \frac{1 + e^{- x_i h_i(x_{pa(i)})}}{
                 1 + e^{2 w_{si} x_s x_i } e^{- x_i h_i(x_{pa(i)})}}}_{=: g_s(x_{bl(s)})}                                            
\end{eqnarray*}
This finally implies 
\begin{equation}
   p_\xi(x_s | x_{bl(s)}) \, = \, \frac{1}{1 + e^{- x_s h_s(x_{pa(s)})} g_s(x_{bl(s)})}.
\end{equation}
Comparing this with the update probability (\ref{updateprob1}), we observe that the full Markov blanket is involved in terms of the modulation
function $g_s$. 
\end{example}

\subsection{Locality of the natural gradient}  
In the previous section, we have evaluated the partial derivatives (\ref{derivative}), which turn out to be local and allow us to apply the (stochastic) gradient method for learning. However, from the information-geometric point of view, we have to use the Fisher-Rao metric for evaluating the gradient, which leads to the well known  
{\em natural gradient method\/}. In general, the gradient is difficult to evaluate because the Fisher information matrix 
has to be inverted (see equations (\ref{gradinloc}) and (\ref{gradinloc2})). In our context of a model that is associated with a directed acyclic graph $G$, however, 
the Fisher information matrix simplifies considerably.   

\begin{theorem} \label{localform}
For a statistical model ${\mathcal M}$ that is parametrised according to (\ref{parametris}), 
the Fisher information matrix  $G({\xi}) := \left( g_{(r;i)(s;j)}({\xi})\right)_{(r;i)(s;j)}$ decomposes into ``local'' $d_r \times d_r$ matrices 
$G_r({\xi}) := \left( g_{(r;i,j)}({\xi}) \right)_{i,j}$, $r \in N$. More precisely, with
\begin{eqnarray}
    \lefteqn{g_{(r; i , j)}(\xi)} \label{localfisher} \\
     & := & \sum_{x_{pa(r)}} p(x_{pa(r)} ; \xi) \sum_{x_r}  k^r(x_r | x_{pa(r)} ; \xi_r) \,
              \frac{\partial  \ln k^r(x_r | x_{pa(r)} ; \cdot) }{\partial \xi_{(r ; i)}}(\xi_r)  \,
              \frac{\partial  \ln k^r(x_r | x_{pa(r)} ; \cdot) }{\partial \xi_{(r ; j)}}(\xi_r) ,  \nonumber
\end{eqnarray}
the following holds:
\[
    g_{(r;i)(s;j)} (\xi) \; = \; 
    \left\{ 
       \begin{array}{c@{,\quad}l}
           g_{(r; i,j)} (\xi) & \mbox{if $r = s$} \\
           0                          & \mbox{otherwise}
       \end{array}   
    \right..
\]
Using matrix notation, 
we have  
\[
G({\xi}) = 
\left(
\begin{array}{ccc}
G_{1}({\xi}) & & 0 \\
  & \ddots &  \\
0 &        & G_{{m+ n}}({\xi})
\end{array}
\right).
\]
\end{theorem}

\begin{proof}
The parametrisation (\ref{parametris}) yields 
\[
\ln  p(x ; \xi) \, = \, \sum_{r \in N} \ln  k^r(x_r | x_{pa(r)} ; \xi_r)
\]
and therefore 
\begin{equation} \label{partdiff}
 \frac{\partial  \ln p(x ; \cdot) }{\partial \xi_{( s ; j )}} (\xi) \, = \, \frac{\partial \ln k^s (x_s | x_{pa(s)} ; \cdot)}{\partial \xi_{( s ; j )}} (\xi_s).
\end{equation}
With (\ref{partdiff}) we obtain for $r \leq s$:
\begin{eqnarray*}
\lefteqn{ g_{(r ; i)(s ; j)}(\xi)} \\ 
      & = & \sum_{x} p(x ; \xi)  \, \frac{\partial  \ln p(x; \cdot) }{\partial \xi_{( r ; i )}} (\xi)  \, \frac{\partial  \ln p(x ; \cdot) }{\partial \xi_{( s ; j )}} (\xi) \\
      & = & \sum_{x} p(x ; \xi)  
           \frac{\partial \ln k^r(x_r | x_{pa(r)}; \cdot)}{\partial \xi_{( r ; i )}}(\xi_r)  \frac{\partial \ln k^s(x_s | x_{pa(s)} ; \cdot)}{\partial \xi_{( s ; j )}}(\xi_s) 
           \qquad\mbox{(by equation (\ref{partdiff}))} \\
      & = & \sum_{x_{< s}} \sum_{x_ s} \sum_{x_{>s}}  
          \left\{\prod_{i < s}  k^i(x_i | x_{pa(i)} ; \xi_i) \right\}   k^s(x_s | x_{pa(s)} ; \xi_s) \left\{\prod_{i > s}  k^i(x_i | x_{pa(i)} ; \xi_i) \right\} \\
      &  &  \qquad \times      
           \frac{\partial \ln k^r(x_r | x_{pa(r)} ; \cdot)}{\partial \xi_{( r ; i )}} (\xi_r)  
           \frac{\partial \ln k^s(x_s | x_{pa(s)} ; \cdot)}{\partial \xi_{( s ; j )}} (\xi_s)   \\
    & = & \sum_{x_{< s}} 
          \left\{\prod_{i < s}  k^i(x_i | x_{pa(i)} ; \xi_i) \right\} \sum_{x_ s}  k^s (x_s | x_{pa(s)} ; \xi_s) 
          \underbrace{ \left\{ \sum_{x_{>s}}  \prod_{i > s}  k^i(x_i | x_{pa(i)} ; \xi_i ) \right\}}_{= 1} \\
      &  &  \qquad \times      
           \frac{\partial \ln k^r (x_r | x_{pa(r)} ; \cdot)}{\partial \xi_{( r ; i )}} (\xi_r) 
           \frac{\partial \ln k^s (x_s | x_{pa(s)} ; \cdot)}{\partial \xi_{( s ; j )}} (\xi_s) \\      
  & = & \sum_{x_{pa(s)}}
        p_\xi(x_{pa(s)})  \sum_{x_ s}  k^s (x_s | x_{pa(s)} ; \xi_s) \,    
           \frac{\partial \ln k^r(x_r | x_{pa(r)} ; \cdot)}{\partial \xi_{( r ; i )}} (\xi_r)  
           \frac{\partial \ln k^s(x_s | x_{pa(s)} ; \cdot)}{\partial \xi_{( s ; j )}} (\xi_s) .
\end{eqnarray*}
If $r \not= s$, this expression reduces to
\begin{eqnarray*}                           
    &   &  \sum_{x_{pa(s)}}
        p_\xi(x_{pa(s)})  
           \frac{\partial \ln k^r(x_r | x_{pa(r)} ; \cdot)}{\partial \xi_{( r ; i )}} (\xi_r) 
            \sum_{x_ s} k^s (x_s | x_{pa(s)} ; \xi_s)  \frac{\partial \ln k^s(x_s | x_{pa(s)} ; \cdot)}{\partial \xi_{( s ; j )}} (\xi_s) \\           
    & = &  \sum_{x_{pa(s)}}
        p_\xi(x_{pa(s)})  
           \frac{\partial \ln k^r (x_r | x_{pa(r)} ; \cdot)}{\partial \xi_{( r ; i )}} (\xi_r) 
            \underbrace{\sum_{x_ s} \frac{\partial k^s (x_s | x_{pa(s)} ; \cdot)}{\partial \xi_{( s ; j )}}(\xi_s)}_{= 0}  \\          
    & = & 0.                   
\end{eqnarray*}
\end{proof}

Theorem \ref{localform} highlights a number of simplifications of the Fisher information matrix as result of the particular parametrisation of the statistical model in terms of a directed acyclic graph. The presented proof is adapted from \cite{Ay02} (see also the related work by \cite{SN17}):  

\begin{enumerate}
\item The Fisher information matrix $G$ has a block structure, reflecting the structure of the underlying graph (see Example \ref{shalldeep}). 
Each block $G_r$ corresponds to a node $r$ and has $d_r \times d_r$ components. Outside these blocks the matrix is filled with zeros. 
The natural gradient method requires the inversion of $G$ (the usual inverse $G^{-1}$, if it exists, or, more generally, the Moore-Penrose inverse $G^+$). 
With the block structure of $G$, this inversion reduces to the inversion of the individual matrices $G_r$. 
The corresponding simplification of the natural gradient is summarised in Corollary \ref{gradcoord}.  

\item The terms $g_{(r; i , j)}(\xi)$, defined by (\ref{localfisher}), are expectation values of the functions 
\[
      C(x_{pa(r)}; \xi_r) \, := \, \sum_{x_r}  k^r(x_r | x_{pa(r)} ; \xi_r) \,
              \frac{\partial  \ln k^r(x_r | x_{pa(r)} ; \cdot) }{\partial \xi_{(r ; i)}}(\xi_r)  \,
              \frac{\partial  \ln k^r(x_r | x_{pa(r)} ; \cdot) }{\partial \xi_{(r ; j)}}(\xi_r) .
\]
These functions are local in two ways. On the one hand, they depend only on local states $x_{pa(r)}$ and, on the other hand, only local parameters $\xi_r$ are involved. This kind of locality is 
very useful in applications of the natural gradient method. Especially in the context of neural networks, locality of learning is considered to be essential. 
Note, however, that the terms $g_{(r; i , j)}(\xi)$ are not completely local. The reason is that the expectation value in (\ref{localfisher}) is taken with respect to $p_\xi$ where $\xi$ is the full parameter vector. 
(As only the distribution of $X_{pa(r)}$ appears, parameters of non-ancestors of $r$ do not play a role in the evaluation of $g_{(r; i , j)}(\xi)$, which simplifies the situation a bit.) In order to evaluate the Fisher information matrix in 
applications, we have to overcome this non-locality by sampling from $p_\xi(x_{pa(r)})$. As we are dealing with directed acyclic graphs, this can be simply done by recursive application of the local kernels $k^r_{\xi_r}$. 
\end{enumerate}

To highlight the relevance of Theorem \ref{localform}, let us consider a few simple examples. 

\begin{example}[Exponential families] \label{exampexpfam}
Consider the statistical model given by local kernels of the exponential form
\begin{equation} \label{expfamstatmod}
    k^r(x_r | x_{pa(r)}; \xi_r) \, = \, \frac{\exp\left( \sum_{i = 1}^{d_r} \xi_{(r;i)} \phi^{(r;i)} (x_{pa(r)}, x_r) \right)}{\sum_{x_r'} \exp\left( \sum_{i = 1}^{d_r} \xi_{(r;i)} \phi^{(r;i)} (x_{pa(r)}, x_r') \right)}.
\end{equation}
In this case, the expression (\ref{localfisher}) yields
\begin{equation} \label{intcov}
   g_{(r;i,j)} \, = \, \sum_{x_{pa(r)}} p( x_{pa(r)} ; \xi) \, {\rm Cov} \Big( \phi^{(r;i)} (x_{pa(r)}, \cdot ) , \phi^{(r; j)} (x_{pa(r)}, \cdot ) \left| x_{pa(r)}; \xi_r \Big)\right. ,
\end{equation}
where the conditional covariance on the RHS of (\ref{intcov}) is evaluated with respect to $k^r( \cdot | x_{pa(r)}; \xi_r)$.
\end{example}

\begin{example}[Neural networks (II)] \label{nn2} 
Neural networks, as introduced in Example \ref{nn1}, 
can be considered as a special case of the statistical models of Example \ref{exampexpfam}. This can be seen by rewriting 
the transition probability (\ref{updateprob1}) as follows: 
\begin{eqnarray*}
       k^r(x_r | x_{pa(r)}; w_r , \vartheta_r)  
       & = & \frac{1}{1 + e^{- x_r h_r(x_{pa(r)})}} \\
       & = &  \frac{e^{\frac{1}{2} x_r h_r(x_{pa(r)})}}{
                  e^{\frac{1}{2} x_r h_r(x_{pa(r)})} +   
                  e^{- \frac{1}{2} x_r h_r(x_{pa(r)})}   }  \\
       & = & \frac{\exp \left( {\frac{1}{2} \sum_{j \in pa(r)} w_{j r} x_j  x_r - \frac{1}{2} \vartheta_r  x_r} \right)}{
                  \sum_{x_r'} \exp \left( {\frac{1}{2} \sum_{j \in pa(r)} w_{j r} x_j  x_r' - \frac{1}{2} \vartheta_r  x_r'} \right) }    
\end{eqnarray*}
This is a special case of (\ref{expfamstatmod}) which only involves pairwise interactions. In order to evaluate the terms (\ref{localfisher}) we need the derivatives
\begin{eqnarray}
   \frac{\partial}{\partial w_{i r}}  \ln k^r(x_r | x_{pa(r)}; w_r , \vartheta_r) 
      & = &  \frac{x_i x_r}{1 + e^{x_r h_r(x_{pa(r)})}} , \label{derivativesyn}  \\
  \frac{\partial}{\partial \vartheta_{r}}  \ln k^r(x_r | x_{pa(r)}; w_r , \vartheta_r) 
      & = & - \frac{x_r}{1 + e^{x_r h_r(x_{pa(r)})}} . \label{derivativethr} 
\end{eqnarray}
According to Theorem \ref{localform}, we can evaluate the Fisher information matrix in a local way. More explicitly, we have
\[
     g_{(r; i , j)}(\xi) \, = \, \sum_{x_{pa(r)}} p_\xi(x_{pa(r)}) \frac{f(x_i, x_j)}{(1+ e^{h_r(x_{pa(r)})})(1+ e^{- h_r(x_{pa(r)})}) } ,
\]   
where 
\[
      f(x_i,x_j) := 
      \left\{
         \begin{array}{c@{,\quad}l}
            x_i x_j & \mbox{if $1 \leq i, j \leq d_r - 1$,} \\
            - x_i     & \mbox{if $1 \leq i \leq d_r - 1$, $j = d_r$,} \\
             - x_j     & \mbox{if $i = d_r$, $1 \leq j \leq d_r - 1$,} \\
             1          & \mbox{if $i = d_r$, $j = d_r$.} \\
         \end{array}
      \right.
\]
\end{example}

\begin{example}[Shallow versus deep networks] \label{shalldeep}
In this example, we demonstrate the difference in sparsity of the Fisher information matrix for architectures of varying depth. 
\begin{figure}[h]        
  \centering
       \includegraphics[width=12cm]{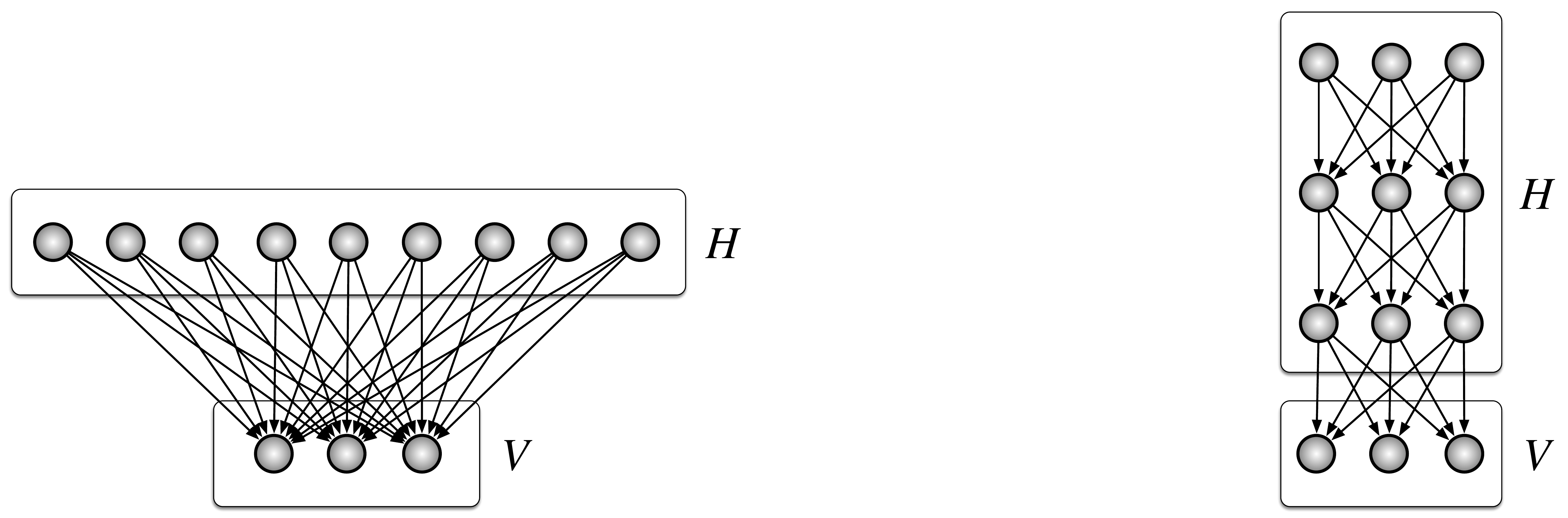}   
  \caption{Two architectures with the same number of parameters but different complexity of the Fisher information matrix.} 
  \label{fig:twodiffarch}
\end{figure}
Figure \ref{fig:twodiffarch} shows two networks with three visible and nine hidden neurons each. The number of synaptic connections is 27 in both cases. 
If we associate one parameter with each edge, the synaptic strength, then we have 27 parameters in the system (for simplicity, we do not consider the threshold values).
Theorem \ref{localform} implies the following structure of the Fisher information matrices in terms of $3 \times 3$ matrices $H_{ij}$ and $G_i$. In the shallow architecture 
we have
\[
     \left(
     \begin{array}{c c c c c c c c c}
     H_{11} &  H_{12} & H_{13} & 0 & 0 & 0 & 0 & 0 & 0 \\
     H_{21} & H_{22} & H_{23} & 0 & 0 & 0 & 0 & 0 & 0 \\
     H_{31} & H_{32} & H_{33} & 0 & 0 & 0 & 0 & 0 & 0 \\
     0      & 0 & 0 & H_{44} & H_{45} & H_{46} & 0 & 0 & 0 \\
     0      & 0 & 0 & H_{54} & H_{55} & H_{56} & 0 & 0 & 0 \\
     0      & 0 & 0 & H_{64} & H_{65} & H_{66} & 0 & 0 & 0 \\
     0      & 0 & 0 & 0 & 0 & 0 & H_{77} & H_{78} & H_{79} \\
      0      & 0 & 0 & 0 & 0 & 0 & H_{87} & H_{88} & H_{89} \\
       0      & 0 & 0 & 0 & 0 & 0 & H_{97} & H_{98} & H_{99} 
       \end{array}
     \right).
\]
The layered architecture implies
\[
     \left(
     \begin{array}{c c c c c c c c c}
     G_1 &      0 & 0 & 0 & 0 & 0 & 0 & 0 & 0 \\
     0      & G_2 & 0 & 0 & 0 & 0 & 0 & 0 & 0 \\
     0      & 0 & G_3 & 0 & 0 & 0 & 0 & 0 & 0 \\
     0      & 0 & 0 & G_4 & 0 & 0 & 0 & 0 & 0 \\
     0      & 0 & 0 & 0 & G_5 & 0 & 0 & 0 & 0 \\
     0      & 0 & 0 & 0 & 0 & G_6 & 0 & 0 & 0 \\
     0      & 0 & 0 & 0 & 0 & 0 & G_7 & 0 & 0 \\
      0      & 0 & 0 & 0 & 0 & 0 & 0 & G_8 & 0 \\
       0      & 0 & 0 & 0 & 0 & 0 & 0 & 0 & G_9 
       \end{array}
     \right).
\]
Out of the $27 \times 27 = 729$ components of the Fisher information matrix, we have $486$ zeros in the shallow case and $648$ zeros in the deep case. 

This example can be generalised to a network with $n$ visible and $m= l \cdot n$ hidden neurons. As in Figure \ref{fig:twodiffarch}, in the one case we arrange all $m$ hidden neurons in one layer of width 
$l \cdot n$ and, in the other case, we arrange the hidden neurons in $l$ layers of width $n$. In both cases, we have $n \cdot m = n (l \cdot n) = l \cdot n^2$ edges, which is the number of parameters, and therefore the Fisher information matrix has $l^2 n^4$ entries. With the shallow architecture, we have at most $n {(l \cdot n)}^2 = l^2 n^3$  non-zero components, whereas in the deep architecture there are at most 
$l \cdot n \cdot n^2 = l \cdot n^3$ non-zero entries. The difference is $l^2 \cdot n^3 - l \cdot n^3 = l \cdot n^3 (l - 1)$. For $n = l = 3$, we recover the above number difference $648 - 486 = 162$.   
\end{example}

\begin{example}[Restricted Boltzmann machine] If we deal with undirected graphical models as models ${\mathcal M}$ on the extended system, then the Fisher information matrix does not necessarily have a block structure as in the case of directed acyclic networks. 
Consider, for instance, a restricted Boltzmann machine, as shown in Figure \ref{fig:Boltzmann1}. With each edge $(i,j) \in V \times H$ we associate a weight $w_{ij}$
and denote the full weight matrix by $W$. The family of all weight matrices parametrises the family
\[
      p(x_V,x_H; W) \, = \, \frac{e^{\sum_{i \in V, j \in H} w_{ij} x_i x_j}}{\sum_{x_V',x_H'} e^{\sum_{i \in V, j \in H} w_{ij} x_i' x_j'}} .
\]  
\begin{figure}[h]        
  \centering
       \includegraphics[width=7cm]{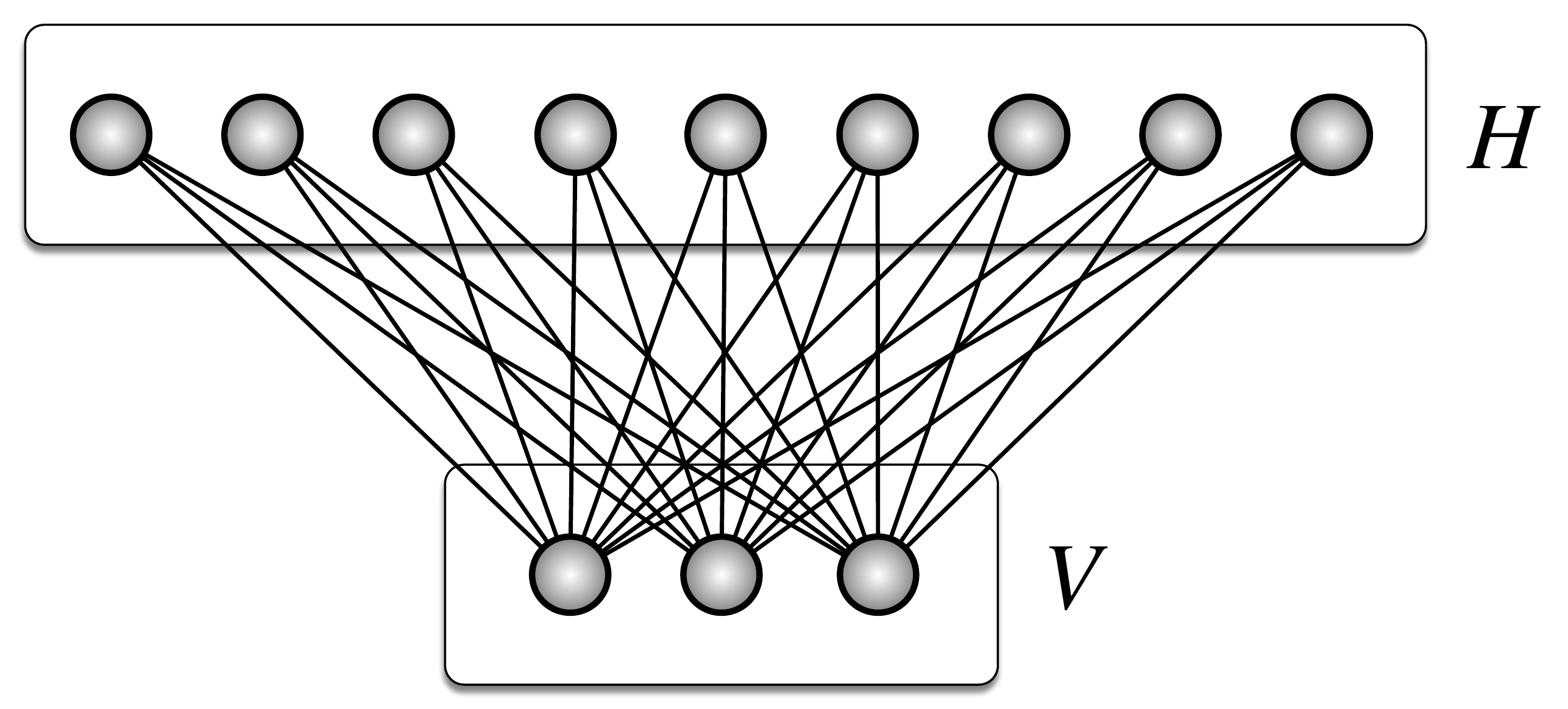}   
  \caption{The architecture of a restricted Boltzmann machine.} 
  \label{fig:Boltzmann1}
\end{figure}
The corresponding marginal model ${\mathcal M}_V$ is called a {\em restricted Boltzmann machine\/}. The Fisher information matrix on ${\mathcal M}$ 
is given by 
\[
     g_{ij,kl} (W) \, = \, {\rm Cov}_{p(\cdot ;W)} \left( X_i X_j , X_k X_l \right)
\]
which has no zeros imposed by the architecture. 
\end{example}
$ $\\

The simplification of the Fisher information matrix, stated in Theorem \ref{localform}, has several important consequences. 
As an immediate consequence we obtain a corresponding simplification of the gradient of a function in terms of 
the parameters 
\[ 
    \xi = (\underbrace{\xi_{(1;1)}, \dots, \xi_{(1;d_1)}}_{ = \xi_1}, \underbrace{\xi_{(2;1)}, \dots, \xi_{(2;d_2)}}_{= \xi_2}, \dots, \underbrace{\xi_{(n+m;1)}, 
    \dots, \xi_{(n+m; d_{n+m})}}_{= \xi_{n + m}}).
\] 
One can interpret these parameters as coordinates, which are then assumed to be in a bijective correspondence with the points of ${\mathcal M}$. 
However, in this article, we explicitly allow for the more general case of an overparametrised model, where the number of parameters can exceed its dimension. This generality 
is particularly important within the context of deep learning. Note that even when the number of parameters coincides with the dimension of 
${\mathcal M}$, we will be dealing with the projected model ${\mathcal M}_V = \pi_V({\mathcal M})$ which is often of lower dimension. Furthermore, the projected model 
can have singular points which do not admit corresponding tangent spaces.    
As a minimal assumption for the study of gradient fields on a general model, 
we will implicitly assume, if not otherwise stated, that the points we are looking at are non-singular. We will revisit the subject of singularities in Section \ref{gensit}.

With the map $\xi \mapsto p_\xi$, 
the tangent space of ${\mathcal M}$ in $p_\xi$ is spanned by the 
vectors $\partial_{(r;i)}(\xi) :=  \frac{\partial}{\partial \xi_{(r;i)}} \, p_\xi$, $r=1,\dots,n+m$, $i = 1,\dots,d_r$. 
We can 
represent the gradient of a smooth function ${E}$ on ${\mathcal M}$, mainly referring to the function (\ref{oriopi}), 
in terms of these tangent vectors:
\begin{equation} \label{localrep}
   {\rm grad}_{\xi} {E} \; = \; \sum_{r = 1}^{n+m} \sum_{i = 1}^{d_r} L_{(r;i)}(\xi) \,  \partial_{(r;i)}(\xi). 
\end{equation}   

\begin{corollary} \label{gradcoord}
Consider the situation of Theorem \ref{localform} and a 
real-valued smooth function $f$ on ${\mathcal M}$. With   
\[
 \nabla_{\xi_r} {E} := 
  \left( 
 \begin{array}{c}
\frac{\partial E}{\partial \xi_{(r;1)}}({\xi}) \\
 \vdots \\
\frac{\partial E}{\partial \xi_{(r;d_r)}}({\xi})
 \end{array}
 \right),
\]
we have the following coordinates of the natural gradient of $E$ in the representation (\ref{localrep}):
\begin{equation}  \label{allgemeineformel2}
 L_r(\xi) \; := \; \left( 
 \begin{array}{c}
   L_{(r;1)}(\xi) \\
 \vdots \\
    L_{(r;d_r)}(\xi)
 \end{array}
 \right) \; = \;  G_r^{+}({\xi}) \, \nabla_{\xi_r} {E} , \qquad r = 1,\dots,n+m.
\end{equation}
Here, $G_r^+(\xi)$ denotes the Moore-Penrose inverse of the matrix $G_r(\xi)$ defined by (\ref{localfisher}). 
(It reduces to the usual matrix inverse whenever $G_r(\xi)$ has maximal rank.)
\end{corollary}
\medskip

Note that Theorem \ref{localform} as well as its Corollary \ref{gradcoord} can equally be applied to the recognition model 
${\mathcal L}_{H|V}$ defined by (\ref{facrec}). In Section 
\ref{wsa} we have studied natural objective functions that involve both, the generative as well as the recognition model, and highlighted their locality properties. Together with the locality of the corresponding Fisher information matrices, these properties allow us to
evaluate a natural gradient version of the wake-sleep algorithm, referred to as {\em natural wake-sleep\/} by \cite{VVMA20}.  
\medskip

The prime objective function to be optimised is typically defined on the projected model ${\mathcal M}_V$ (see, e.g., the function (\ref{minimisation})). 
It naturally carries the Fisher-Rao metric of ${\mathcal P}_V$
so that we can define the natural gradient of the given objective function 
directly on ${\mathcal M}_V$. On the other hand, we have seen that the Fisher information matrix on the full model 
${\mathcal M} \subseteq {\mathcal P}_{V,H}$ 
has a block structure associated with the underlying network. 
This implies useful locality properties of the natural gradient and thereby makes the method applicable within the context of deep learning.  
The main problem that we are now going to study is the following: Can we extend the locality of the natural gradient on the full model ${\mathcal M}$, 
as stated in Corollary \ref{gradcoord}, to the natural gradient on the projected model ${\mathcal M}_V$? 
In the following section we first study this problem in a more general setting of Riemannian manifolds.

\section{Gradients on full versus coarse grained models}  \label{singularnatgrad} 
\subsection{The general problem} \label{gensit}
We now develop a more general perspective, which we motivate by analogy to the context of the previous sections.  
Assume that we have two 
Riemannian manifolds $({\mathcal Z}, g^{\mathcal Z})$ and 
$({\mathcal X}, g^{\mathcal X})$ 
and a differentiable map $\pi: {\mathcal Z} \to {\mathcal X}$, with its differential $d \pi_p: T_p{\mathcal Z} \to T_{\pi(p)}{\mathcal X}$ 
in $p$. 
The manifold ${\mathcal Z}$ corresponds to the manifold of (strictly positive) distributions on the full set of units, the visible and the hidden units. The map 
$\pi$ plays the role of the marginalisation map which marginalises out the hidden units and which we will interpret in Section \ref{chents}
as one instance of a more general coarse graining procedure.  
Typically, we have a model ${\mathcal M} \subseteq {\mathcal Z}$ which corresponds to 
a model consisting of the joint distributions on the full system that can be represented by the network. It is obtained in terms of a parametrisation 
$\varphi: \Xi \to {\mathcal Z}$, $\xi \mapsto p_\xi$, where $\Xi$ is a differentiable manifold, usually an open subset of ${\Bbb R}^d$. 
In general, ${\mathcal M}$ will not be a sub-manifold of ${\mathcal Z}$ and 
can contain various kinds of singularities (for more details see \cite{Wa09}). 
We restrict attention to the non-singular points of ${\mathcal M}$. A point 
$p$ in ${\mathcal M} \subseteq {\mathcal Z}$ is said to be a
{\em non-singular\/} point of ${\mathcal M}$ if there exists a smooth chart $\psi: U \to V \subseteq {\Bbb R}^n$ for ${\mathcal Z}$ 
such that $p \in U$ and, for some $k$,  
\begin{equation} \label{slice}
   \psi({\mathcal M} \cap U) = \{ (x_1,\dots,x_k, x_{k +1}, \dots , x_{n}) \in V \, : \,  x_{k +1} = \cdots =  x_{n} = 0\}.
\end{equation}
We denote the set of non-singular points of ${\mathcal M}$ by ${\rm Smooth}({\mathcal M})$. 
If a point $p \in {\mathcal M}$ is not non-singular, it is called a {\em singularity\/} or a {\em singular point\/} of ${\mathcal M}$. 
In a non-singular point $p$, the tangent space $T_p {\mathcal M}$ is well defined. Throughout this article, 
we will assume that the parametrisation $\varphi$ of 
${\mathcal M}$ is a {\em proper parametrisation\/} in the sense that
for all $p \in {\rm Smooth}({\mathcal M})$ and all $\xi \in \Xi$ with $\varphi(\xi) = p$, 
the image of the differential ${d \varphi}_\xi$ coincides with the full tangent space $T_{p} {\mathcal M}$. This assumption 
is required, but often not explicitly stated, when dealing with the natural gradient method for optimisation on parametrised models. 
More precisely, 
when we interpret the Fisher information matrix (\ref{fisherinf})
as a ``coordinate representation'' of the Fisher-Rao metric, we implicitly assume that the vectors $\partial_i (\xi) = \frac{\partial}{\partial \xi_i} p_\xi$, $i = 1,\dots, d$, span the tangent space of the model in $p_\xi$. Note that linear independence, which ensures the 
non-degeneracy of the Fisher information matrix, is not required and would in fact be too restrictive given that overparametrised models
play an important role within the field of deep learning. 
\medskip

We now consider a smooth function $f: {\mathcal X} \to {\Bbb R}$ and study its gradient on ${\mathcal X}$ 
(with respect to $g^{\mathcal X}$) in relation to the corresponding gradient of $f \circ \pi: {\mathcal Z} \to {\Bbb R}$ on 
${\rm Smooth}({\mathcal M})$ (with respect to $g^{\mathcal Z}$). For a non-singular point of ${\mathcal M}$, we decompose 
the tangent space $T_p {\mathcal M}$ into a ``vertical component''
$T^{\mathcal V}_p {\mathcal M} := T_p {\mathcal M} \cap \ker {d {\pi}}_p$ and its orthogonal complement 
$T^{\mathcal H}_p {\mathcal M}$ in $T_p {\mathcal M}$, the corresponding ``horizontal component''. 
We have the following proposition where we use the somewhat simpler notation ``$\langle \cdot , \cdot \rangle$'' for both metrics, 
$g^{\mathcal Z}$ and $g^{\mathcal X}$. 

\begin{proposition} \label{general}
Consider a model ${\mathcal M}$ in ${\mathcal Z}$ and a differentiable map $\pi : {\mathcal Z} \to {\mathcal X}$ and let    
$p$ be a non-singular point of ${\mathcal M}$. 
Assume that the following consistency condition is satisfied:
\begin{equation} \label{invaria}
X,Y \in T^{\mathcal H}_p {\mathcal M}  \quad \Rightarrow \quad 
{\left\langle X,Y \right\rangle}_p = {\left\langle {d {\pi}}_p (X), {d {\pi}}_p (Y) \right\rangle}_{{\pi}(p)}.
\end{equation}
Then, for all smooth functions $f: {\mathcal X} \to {\Bbb R}$, we have 
\begin{equation} \label{pushgrad}
     {d \pi}_p \left( {\rm grad}^{\mathcal M}_p (f \circ \pi) \right) =  \Pi  \left( {\rm grad}^{\mathcal X}_{\pi (p)} f \right) ,
\end{equation}
where $\Pi$ denotes the projection of tangent vectors in $T_{\pi(p)} {\mathcal X}$ onto ${d {\pi}}_p (T_p {\mathcal M})$.   
\end{proposition}    
\begin{proof}
First observe that ${\rm grad}^{\mathcal M}_p (f \circ \pi) \in T^{\mathcal H}_p {\mathcal M}$. 
Indeed, for all $Y \in T^{\mathcal V}_p {\mathcal M} $ we have 
\begin{equation} \label{orth}
 \langle {\rm grad}^{\mathcal M}_p (f \circ \pi), Y \rangle \; = \; {d (f \circ \pi)}_p (Y) \; = \; d f_{\pi(p)} (\underbrace{{d \pi}_p (Y)}_{= 0}) \; = \; 0.
\end{equation}
Let $X' \in {d \pi}_p(T_p {\mathcal M}) \subseteq  T_{\pi(p)} {\mathcal X}$. There exists $X \in T_p {\mathcal M}$
such that ${d \pi}_p (X) = X'$. We can decompose $X$ orthogonally into a part $X_1$ contained in 
$T^{\mathcal V}_p {\mathcal M}$ and a part $X_2$ contained in $T^{\mathcal H}_p {\mathcal M}$. 
With this decomposition we have $X' = {d \pi}_p (X) = {d \pi}_p (X_1 + X_2) = {d \pi}_p (X_2)$.
This implies
\begin{eqnarray*} 
   {\langle {d \pi}_p ({\rm grad}^{\mathcal M}_p (f \circ \pi)) , X' \rangle}_{\pi(p)} 
      & = &    {\langle {d \pi}_p ({\rm grad}^{\mathcal M}_p (f \circ \pi)) ,  {d \pi}_p (X_2) \rangle}_{\pi(p)} \\
      & = & {\langle {\rm grad}^{\mathcal M}_p (f \circ \pi) , X_2 \rangle}_p \qquad (\mbox{because of (\ref{orth}) and (\ref{invaria})}) \\
      & = & {d (f \circ \pi)}_p (X_2) \\
      & = & {d f}_{\pi(p)} ({d \pi}_p (X_2)) \\
      & = & {\langle {\rm grad}^{\mathcal X}_{\pi(p)} f , {d \pi}_p (X_2) \rangle}_{\pi(p)} \\
      & = & {\langle {\rm grad}^{\mathcal X}_{\pi(p)} f , X' \rangle}_{\pi(p)} .
\end{eqnarray*}
This proves equation (\ref{pushgrad}). 
\end{proof}

As stated above, the parametrised model ${\mathcal M}$ plays the role of 
the distributions on the full network, consisting of the visible and hidden units. We want to relate this 
model to the projected model ${\mathcal S} := \pi({\mathcal M})$. 
The composition of the parametrisation $\varphi$ and the projection $\pi$ 
serves as a parametrisation $\xi \mapsto \pi (p_{\xi} )$ of ${\mathcal S}$ as shown in the following diagram.  
\[
\begin{tikzcd}
 \Xi    \arrow{rr}{\varphi} \arrow[rrrr, bend left, "\pi \circ \varphi"] & &  {\mathcal M} \arrow{rr}{\pi} 
         & &  {\mathcal S}  \\                                                             
\end{tikzcd}
\]
The map $\pi \circ \varphi$ is a proper parametrisation ${\mathcal S}$ if for all $q \in {\rm Smooth}({\mathcal S})$ 
and all $\xi$ with $\pi(p_\xi) = q$, the image of the differential ${d (\pi \circ \varphi)}_\xi$ coincides with the full tangent space $T_{q} {\mathcal S}$.   
Obviously, this does not follow from the assumption that $\varphi$ is a proper parametrisation of ${\mathcal M}$ and requires further assumptions. 
One necessary, but not sufficient, condition 
is the following: Assume that $\pi \circ \varphi$ is a proper parametrisation of ${\mathcal S}$ and consider a point  
$p \in {\rm Smooth}({\mathcal M})$ with $\pi(p) \in {\rm Smooth}({\mathcal S})$.    
With $\xi \in \Xi$, $\varphi(\xi) = p$, we have
\begin{eqnarray}
T_{\pi(p)} {\mathcal S} & = & {d (\pi \circ \varphi)}_\xi \left( T_\xi \Xi \right) \qquad \;\,\, (\mbox{$\pi \circ \varphi$ proper parametristation}) \nonumber \\
   & = & {d \pi}_{\varphi(\xi)} \left( {d \varphi}_\xi \left( T_\xi \Xi \right) \right) \;\, \, \, \, \, (\mbox{by the chain rule}) \nonumber \\
   & = & {d \pi}_{\varphi(\xi)} \left( T_{\varphi(\xi)} {\mathcal M} \right) \qquad (\mbox{$\varphi$ proper parametristation}) \nonumber \\
   & = & {d \pi}_{p} \left( T_{p} {\mathcal M} \right). \label{surj}
\end{eqnarray}
The condition (\ref{surj}) is sufficient if $\pi^{-1}({\rm Smooth}({\mathcal S})) \subseteq {\rm Smooth}({\mathcal M})$, which is clearly 
satisfied if ${\mathcal M}$ is a smooth sub-manifold of ${\mathcal Z}$ (with no singularities). 
\medskip

We have the following implication of Proposition \ref{general}.

\begin{theorem} \label{theinvargrad}
Consider a proper parametrisation $\varphi: \Xi \to {\mathcal Z}$ of ${\mathcal M}$ 
and a smooth map $\pi: {\mathcal Z} \to {\mathcal X}$.
Furthermore, assume that the compatibility condition (\ref{invaria}) is satisfied.  
If the composition $\pi \circ \varphi: \Xi \to {\mathcal X}$ is a proper parametrisation of ${\mathcal S} = \pi({\mathcal M})$ then 
for all $p \in {\rm Smooth}({\mathcal M})$ with $\pi(p) \in {\rm Smooth} ({\mathcal S})$, 
and all smooth functions $f : {\mathcal X} \to {\Bbb R}$, we have 
\begin{equation} \label{identitaet11}
      {d \pi}_{p} \left( {\rm grad}^{\mathcal M}_{p} (f \circ \pi) \right) =  {\rm grad}^{\mathcal S}_{\pi(p)} f .  
\end{equation}
\end{theorem}
\begin{proof} The assumption that $\pi \circ \varphi$ is a proper parametrisation implies $T_{\pi(p)} {\mathcal S} = {d \pi}_{p} \left( T_{p} {\mathcal M} \right)$ 
(see (\ref{surj})). In that case, the projection $\Pi$ on the RHS of (\ref{pushgrad}) reduces to the identity map. 
\end{proof}

Note that if we do not assume that the composition $\pi \circ \varphi$ is a proper parametrisation of ${\mathcal S}$ we 
have to replace the RHS of (\ref{identitaet11}) by 
$\Pi \left( {\rm grad}^{\mathcal S}_{\pi(p)} f \right)$, where $\Pi$ denotes the projection of tangent vectors in 
$T_{\pi(p)} {\mathcal S}$ onto ${d {\pi}}_p (T_p {\mathcal M})$. Therefore, without a proper parametrisation 
it can well be the case 
that the gradient on ${\mathcal M}$ vanishes in a point $p$ while the corresponding gradient on ${\mathcal S}$, 
that is ${\rm grad}^{\mathcal S}_{\pi (p)} f$, does not. 
Such a point $p$ is referred to as spurious critical point (see \cite{TKB20}). In addition to the problem of having singularities of ${\mathcal S} = \pi({\mathcal Z})$, this 
represents another problem with gradient methods for the optimisation of smooth functions on parametrised models. However, it turns out that in the context of the natural gradient method, where we require models to be properly parametrised, the problem of spurious critical points does not appear. 
\medskip

We conclude this section by addressing the following problem: If we assume 
that the compatibility condition (\ref{invaria}) is satisfied for a model ${\mathcal M}$ in ${\mathcal Z}$, 
what can we say about the corresponding compatibility for a sub-model ${\mathcal L}$ of ${\mathcal M}$?  
In general we cannot expect that (\ref{invaria}) also holds for ${\mathcal L}$. 
The following theorem characterises those sub-models ${\mathcal L}$ of ${\mathcal M}$ for which this is satisfied.  

\begin{theorem} \label{giltauch}
Assume that (\ref{invaria}) holds for a model ${\mathcal M}$ in ${\mathcal Z}$ and consider a sub-model 
${\mathcal L} \subseteq {\mathcal M}$. Then (\ref{invaria}) also holds for ${\mathcal L}$ if and only if 
for each point $p \in {\rm Smooth} ({\mathcal L})$ 
the tangent space $T_p {\mathcal L}$ satisfies 
\begin{equation} \label{conditionforsub}
   T_p {\mathcal L} \, = \, \left( T_p {\mathcal L}  \cap T^{\mathcal H}_p {\mathcal M} \right) + 
   \left(T_p {\mathcal L}  \cap T^{\mathcal V}_p {\mathcal M}  \right). 
\end{equation}
\end{theorem} 

This theorem is a direct implication of Lemma \ref{lemmlin} below which reduces the problem to the simple setting of linear algebra. 
Let $({B} , {\langle \cdot , \cdot \rangle}_{B})$, $({C}, {\langle \cdot , \cdot \rangle}_{C})$ be two 
finite-dimensional real Hilbert spaces, and let $T: {B} \to {C}$ be a linear map. We can decompose 
${B}$ into a ``vertical component'' ${B}^{\mathcal V} := \ker T$ and its orthogonal complement 
${B}^{\mathcal H}$
in $B$, the corresponding ``horizontal component''.  
Now let ${A}$ be a linear subspace of ${B}$, equipped with the induced inner product ${\langle \cdot, \cdot \rangle}_A$, 
and consider the restriction $T_{A}: A \to C$ of $T$ to $A$. 
Denoting by $\bot_{A}$ and $\bot_{B}$ the orthogonal complements in ${A}$ and ${B}$, respectively, 
we can decompose ${A}$ into 
\begin{equation}
 {A}^{\mathcal V} \, := \, \ker T_{A} \, = \, {A} \cap \ker T \, = \, A \cap B^{\mathcal V} ,
\end{equation}
and
\begin{equation}   
  {A}^{\mathcal H} \, := \, \left( {A}^{\mathcal V} \right)^{\perp_{A}} \, = \, A \cap \left( {A}^{\mathcal V} \right)^{\perp_{B}} \, = \, 
  A \cap \left( A \cap B^{\mathcal V} \right)^{\perp_{B}} \, = \,  A \cap \left( A^{\perp_{B}} + B^{\mathcal H} \right) \, \supseteq \, A \cap B^{\mathcal H} .
\end{equation} 
Note that, while we always have ${A}^{\mathcal V} \subseteq {B}^{\mathcal V}$, in general ${A}^{\mathcal H} \not\subseteq {B}^{\mathcal H}$.

\begin{lemma} \label{lemmlin}
Assume:
\begin{equation} \label{lininv}
   X, Y \in B^{\mathcal H} \quad \Rightarrow \quad {\langle X , Y \rangle}_{B} = {\langle T(X) , T(Y) \rangle}_{C}.
\end{equation}
Then the following two statements about a subspace $A$ of $B$ are equivalent: 
\begin{eqnarray} 
   (i) & X, Y \in A^{\mathcal H} \quad \Rightarrow \quad {\langle X , Y \rangle}_{A} = 
         {\langle T_{A} (X) , T_{A} (Y) \rangle}_{C}. \label{39aufu} \\
  (ii) & {A} \, = \, ({A} \cap B^{\mathcal H})  + ({A} \cap B^{\mathcal V}). \label{decomp}
\end{eqnarray}
\end{lemma}

\begin{proof} Let us first assume that (\ref{decomp}) holds true. 
This implies 
\begin{equation} 
   A^{\mathcal H} \, = \, \left( A^{\mathcal V} \right)^{\perp_{A}} \, = \, \left( A \cap B^{\mathcal V} \right)^{\perp_{A}} \, = \, 
   A \cap B^{\mathcal H} \, \subseteq \, B^{\mathcal H} .
\end{equation}
For all $X, Y \in A^{\mathcal H}  \subseteq  B^{\mathcal H}$,  (\ref{lininv}) then takes the form
\begin{equation} \label{lininv2}  
   {\langle X , Y \rangle}_{A} =   {\langle X , Y \rangle}_{B} =  {\langle T(X) , T(Y) \rangle}_{C} 
   =  {\langle T_{A}(X) , T_{A}(Y) \rangle}_{C}.
\end{equation}
In order to prove the opposite implication, we assume that (\ref{decomp}) does not hold for ${A}$. This means that 
\begin{equation} \label{propsubset}
     {Q} := ({A} \cap B^{\mathcal H} ) +  ( {A} \cap B^{\mathcal V} )
\end{equation}
is a proper subspace of ${A}$. 
We denote the orthogonal complement of ${Q}$ in 
${A}$ by ${R}$ and choose a non-trivial vector $X$ in ${R}$. Such a vector can be uniquely decomposed as a sum of two non-trivial 
vectors $X_1 \in B^{\mathcal H}$ and $X_2 \in B^{\mathcal V} $. This implies 
\begin{eqnarray*}
  {\| X \|}_{A} & = &  {\| X \|}_{B} \\
    & = &  {\| X_1 + X_2 \|}_{B} \\
    & > &   {\| X_1 \|}_{B} \\
    & = &  {\| T(X_1) \|}_{C} \\
    & = &  {\| T(X_1) + T(X_2) \|}_{C} \\
    & = &  {\| T(X) \|}_{C} \\
    & = &  {\| T_{A}(X) \|}_{C} .
\end{eqnarray*}
This means that (\ref{39aufu}) does not hold for the subspace ${A}$. 
\end{proof}

\subsection{A new interpretation of Chentsov's theorem} \label{chents}
We now come back to the context of probability distributions but take a slightly more general perspective than in Section \ref{xysetting}.    
We interpret $X_V$ as a coarse graining of the set $\mathsf{X}_V \times \mathsf{X}_H$ which lumps together all pairs $(v , h)$, $(v', h')$ with $v = v'$. 
Replacing the Cartesian product $\mathsf{X}_V \times \mathsf{X}_H$ by a general set $\mathsf{Z}$,
a coarse graining of $\mathsf{Z}$ is an onto mapping $X: \mathsf{Z} \to \mathsf{X}$, which partitions $\mathsf{Z}$ into the atoms 
$\mathsf{Z}_x := X^{-1}(x)$. The corresponding push-forward map is given by
\[
   {X}_\ast: {\mathcal P}(\mathsf{Z}) \to  {\mathcal P}(\mathsf{X}), \qquad p = 
   \sum_{z} p(z) \, \delta^z \mapsto p_X = \sum_x \underbrace{\left( \sum_{z \in \mathsf{Z}_x} p(z) \right)}_{=: p(x)} \delta^x ,
\]
with the differential 
\[
   {d X_\ast}_p: {\mathcal T}(\mathsf{Z}) \to  {\mathcal T}(\mathsf{X}), \qquad V = \sum_{z} V(z) \, \delta^z \mapsto  \sum_x \left( \sum_{z \in \mathsf{Z}_x} V(z) \right) \delta^x .
\]
Obviously, we have 
\begin{equation} \label{vertik}
     {\mathcal V}_p \, := \, \ker {d X_\ast}_p \, = \, 
     \left\{ \sum_{z} V(z) \, \delta^{z}  \; : \; \mbox{$\sum_{z \in \mathsf{Z}_x} V(z) \, = \, 0$ for all $x$} \right\} ,
\end{equation}
with the orthogonal complement 
\begin{equation} \label{hori}
    {\mathcal H}_p \, := \, {{\mathcal V}_p}^\perp   \, = \, 
    \left\{  \widetilde{U} = \sum_{x} \frac{U(x)}{p(x)} \sum_{z \in \mathsf{Z}_x} p(z) \, \delta^z   \; : \; \sum_x U(x) = 0 \right\}
\end{equation}
with respect to the Fisher-Rao metric in $p$ (note that ${\mathcal V}_p$ is independent of $p$). For a vector 
\[
    \widetilde{U} \, = \,  \sum_{x} \frac{U(x)}{p(x)} \sum_{z \in \mathsf{Z}_x} p(z) \, \delta^z \; \in \;  {\mathcal H}_p, 
\] 
we have
\[
    {d X_\ast}_p (\widetilde{U}) \, = \, U.
\]
Given a vector $V = \sum_{z} V(z) \, \delta^z \in {\mathcal T}(\mathsf{Z})$, we can decompose it uniquely as 
\[
    V = V^{\mathcal H} + V^{\mathcal V}, 
\]
with $V^{\mathcal H} \in {\mathcal H}_p$ and $V^{\mathcal V} \in {\mathcal V}_p$. More precisely,
\begin{eqnarray}
      V^{\mathcal H} & = &  \sum_x \sum_{z \in \mathsf{Z}_x} \left( \frac{p(z)}{p(x)}\sum_{z' \in \mathsf{Z}_x} V(z') \right) \delta^z , \\
      V^{\mathcal V} & = & \sum_x \sum_{z \in \mathsf{Z}_x} \left( V(z) - \frac{p(z)}{p(x)}\sum_{z' \in \mathsf{Z}_x} V(z') \right) \delta^z .
\end{eqnarray}
We now examine the inner product of two such vectors $\widetilde{U}, \widetilde{V} \in {\mathcal H}_p$:
\begin{eqnarray}
{\langle \widetilde{U},  \widetilde{V} \rangle}_{p}  
       & = & \sum_{z} \frac{1}{p(z)} \, \widetilde{U}(z) \widetilde{V}(z) \nonumber \\
       & = & \sum_x \sum_{z \in Z_x} \frac{1}{p(z)} \left( \frac{U(x)}{p(x)} p(z)  \right) \left(  \frac{V(x)}{p(x)} p(z) \right) \nonumber \\
       & = & \sum_x   \frac{U(x)}{p(x)}   \frac{V(x)}{p(x)} \sum_{z \in Z_x} p(z) \nonumber \\
       & = &  \sum_x \frac{1}{p(x)} \, U(x)  V(x)  \nonumber \\
       & = & {\langle U , V \rangle}_{p_X} . \label{innprod}
\end{eqnarray}
As the inner product (\ref{innprod}) coincides with ${\langle   {d X_\ast}_p (\widetilde{U}) ,  {d X_\ast}_p (\widetilde{V}) \rangle}_{X_\ast(p)}$, 
the compatibility condition (\ref{invaria}) is satsified.  
Given that there are no singularities involved, 
Theorem \ref{theinvargrad}
implies that for all smooth functions $f: {\mathcal P}(\mathsf{X}) \to {\Bbb R}$ 
and all $p \in {\mathcal P}(\mathsf{Z})$, the following equality of gradients holds:
\begin{equation} \label{identitaet3}
      {d X_\ast}_p \left( {\rm grad}_{p} (f \circ X_\ast) \right) \, = \, {\rm grad}_{X_\ast(p)} f ,
\end{equation}
where ${\mathcal P}(\mathsf{Z})$ and ${\mathcal P}(\mathsf{X})$ are equipped with the respective Fisher-Rao metrics. 
Even though this is a simple observation, it highlights an important point here. A coarse graining is generally associated 
with a loss of information, which is expressed by the monotonicity of the Fisher-Rao metric. This information loss is maximal when we project from the full space 
${\mathcal P}(\mathsf{Z})$ onto ${\mathcal P}(\mathsf{X})$. 
Nevertheless, the gradient of any function $f$ that is defined on ${\mathcal P}(\mathsf{X})$ is not sensitive to this information loss. 
In order to study parametrised models ${\mathcal M}$ in ${\mathcal P}(\mathsf{Z})$ with the 
same invariance of gradients, we have to impose the condition (\ref{conditionforsub}), which takes the form 
\begin{equation} \label{admissib}
     T_p {\mathcal M} \, = \, \big( T_p {\mathcal M} \cap {\mathcal H}_p \big) +  \big( T_p {\mathcal M} \cap  {{\mathcal V}_p} \big), 
     \qquad p \in {\rm Smooth}({\mathcal M}). 
\end{equation}

\begin{definition} \label{cylindrical}
If a model ${\mathcal M} \subseteq {\mathcal P}(\mathsf{Z})$ satisfies the condition (\ref{admissib}) in $p$, we say that it is 
{\em cylindrical\/} in $p$. If it is cylindrical in all non-singular points, we say that it is 
({\em pointwise\/}) {\em cylindrical\/}.  
\end{definition}

Of particular interest are cylindrical models with a trivial vertical component. 
These are the models, for which the coarse graining $X$ is a minimal sufficient statistic. 
They have been used by \cite{Che82} in order to characterise the Fisher-Rao metric. To be more precise, 
we need the definition of a {\em Markov kernel\/}. We consider the space 
of linear maps from ${\mathcal Z} = {\Bbb R}^{\mathsf{Z}}$ to ${\mathcal X} = {\Bbb R}^{\mathsf{X}}$, which is canonically isomorphic to 
${\mathcal Z}^\ast \otimes {\mathcal X}$, and define the polytope of Markov kernels as  
\[
    {\mathcal K}({\mathsf{Z}} | {\mathsf{X}}) \, := \, 
    \left\{ K = \sum_{x,z} k(z | x) \, \delta^z \otimes e_x \; : \; \mbox{$k(z | x) \geq 0$ for all $x,z$, and $\displaystyle \sum_z k(z | x) = 1$ for all $x$} \right\}.
\]  
The set ${\mathcal P}(\mathsf{Z})$ of probability vectors is a subset where each vector $p$ is identified 
with $k(z | x) := p(z)$. We have equality of the two sets if $\mathsf{X}$ consists of only one element.
We now consider a Markov kernel $K$ that is coupled with the coarse graining $X: \mathsf{Z} \to \mathsf{X}$ in the sense that it 
satisfies $k(z | x) > 0$ if and only if $z \in \mathsf{Z}_x$.
This defines an embedding $K_\ast: {\mathcal P}(\mathsf{X}) \; \to \; {\mathcal P}(\mathsf{Z})$, 
\[
    \qquad p = \sum_x p(x) \, \delta^x \; \mapsto \; \sum_{z} \left(\sum_x p(x) k(z | x) \right) \delta^z = \sum_{x} p(x) \left(\sum_{z \in \mathsf{Z}_x} k(z | x) \, \delta^z\right).  
\] 
The image of $K_\ast$, which we denote by ${\mathcal M}(K)$, is a simplex, given by the extreme points
\[
     \sum_{z \in \mathsf{Z}_x} k(z | x) \, \delta^z, \qquad x \in \mathsf{X},
\]
and we have $X_\ast \circ K_\ast = {\rm id}_{\mathcal{P}(\mathsf{X})}$. The differential of $K_\ast$ is given by 
\[
   {d K_\ast}: {\mathcal T}(\mathsf{X}) \; \to \; {\mathcal T}(\mathsf{Z}), \qquad V = \sum_x V(x) \, \delta^x \; \mapsto \; \sum_{x} V(x) \left(\sum_{z \in \mathsf{Z}_x} k(z | x) \, \delta^z\right),  
\] 
with image 
\[
   {\rm im} \, d K_\ast \, = \, \left\{  \sum_{x} V(x) \left(\sum_{z \in \mathsf{Z}_x} k(z | x) \, \delta^z\right)   \; : \; \sum_x V(x) = 0 \right\}. 
\]
The following simple calculation shows that $K_\ast$ is an isometric embedding, referred to as {\em Markov embedding\/} (see Figure \ref{fig:MarkovEmb}): 
\begin{eqnarray}
\lefteqn{{\langle d K_\ast (U),  d K_\ast (V) \rangle}_{K_\ast(p)}} \nonumber \\ 
       & = & \sum_x \sum_{z \in \mathsf{Z}_x} \frac{1}{\sum_{x'} p(x') k(z | x')} \left( \sum_{x'} U(x') k(z | x') \right) 
                  \left( \sum_{x'} V(x') k(z | x') \right)  \nonumber  \\
       & = &  \sum_x \sum_{z \in \mathsf{Z}_x} \frac{1}{p(x) k(z | x)} U(x) k(z | x)  V(x) k(z | x)  \nonumber  \\
       & = &  \sum_x \frac{1}{p(x)} U(x)  V(x) \sum_{z \in {\mathsf Z}_x}  k(z | x)  \nonumber  \\
       & = &  \sum_x \frac{1}{p(x)} U(x)  V(x)  \nonumber  \\
       & = & {\langle U , V \rangle}_{p}. \label{iso}
\end{eqnarray}
Obviously, 
for $p \in {\mathcal M}(K)$, we have $T_p {\mathcal M}(K) = {\mathcal H}_p$, and therefore 
$T_p {\mathcal M} \cap {\mathcal H}_p = {\mathcal H}_p$ and $T_p {\mathcal M} \cap  {{\mathcal V}_p} = \{0\}$. This 
implies (\ref{admissib}) and thereby proves that ${\mathcal M}(K)$ is cylindrical. In analogy to (\ref{identitaet3}),  
we have for all smooth functions $f: {\mathcal P}(\mathsf{X}) \to {\Bbb R}$ and all 
$p \in {\mathcal M}(K)$, 
\begin{equation} \label{identitaet4}
      {d X_\ast}_p \left( {\rm grad}^{{\mathcal M}(K)}_{p} (f \circ X_\ast) \right) \, = \, {\rm grad}_{X_\ast(p)} f , 
\end{equation}
where this time the gradient on the LHS is evaluated on ${\mathcal M}(K)$, with respect to the induced Fisher-Rao metric, and the one on the RHS remains as it is.  
This is a simple observation which follows directly from the fact that $\left. X_\ast \right|_{{\mathcal M}(K)}$ is an isometry 
between ${\mathcal M}(K)$ and ${\mathcal P}(\mathsf{X})$ (see (\ref{iso})). In fact, $\left. X_\ast \right|_{{\mathcal M}(K)}$ being an isometry is equivalent to 
the invariance (\ref{identitaet4}) of the gradients. 

\begin{figure}[h]        
  \centering
       \includegraphics[width=12cm]{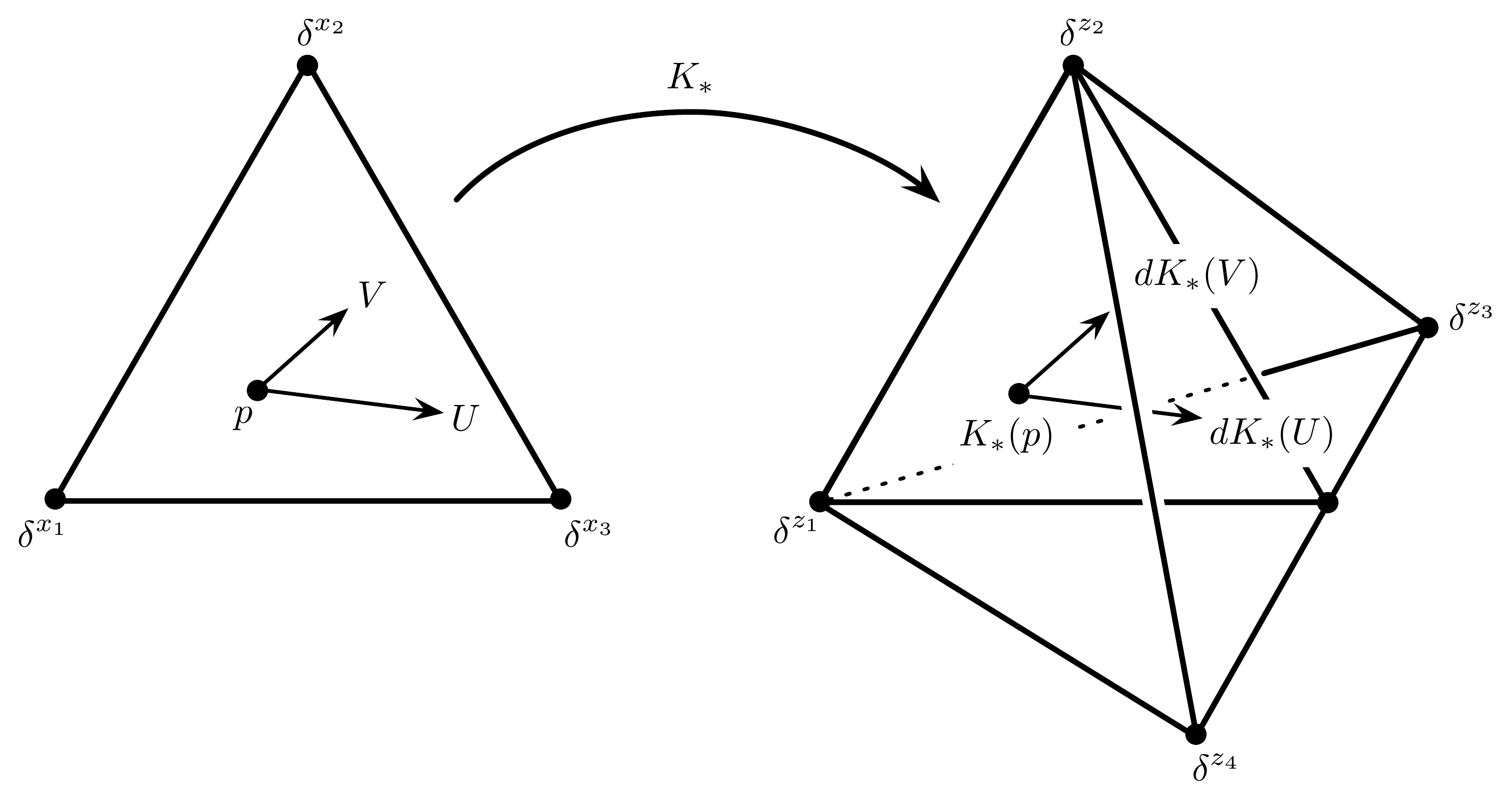}   
  \caption{Markov embedding associated with the following coarse graining $X$: $z_1 \mapsto x_1$, $z_2 \mapsto x_2$, 
  $z_3 \mapsto x_3$, $z_4 \mapsto x_3$. The inner product between $U$ and $V$ equals the inner product of 
  $dK_\ast(U)$ and $d K_\ast(V)$ (see (\ref{iso})).}
  \label{fig:MarkovEmb}
\end{figure}

In order to compute the gradient of a function on an extended space that is equivalent to the actual gradient, we want to use 
Theorem \ref{theinvargrad}. Its applicability is based on the invariance property (\ref{innprod}) of the Fisher-Rao metric with respect to 
coarse grainings.
Instances of this equivalence are given by the equations (\ref{identitaet3}) and (\ref{identitaet4}) where we considered 
two extreme cases, the full model ${\mathcal P}(\mathsf{Z})$ and the model ${\mathcal M}(K)$, respectively, which both project onto 
${\mathcal P}(\mathsf{X})$. We know that Theorem \ref{theinvargrad} also holds for all cylindrical models ${\mathcal M}$, 
including, but not restricted to, intermediate cases where ${\mathcal M}(K) \subseteq {\mathcal M} \subseteq {\mathcal P}(\mathsf{Z})$.  
How flexible are we here with the choice of the metric? 
In fact, 
a reformulation of Chentsov's uniqueness result identifies the Fisher-Rao metric as the only metric for which 
Theorem \ref{theinvargrad} holds. 

\begin{theorem} \label{chentsovreform} 
Assume that for any non-empty finite set $\mathsf{S}$, ${\mathcal P}(\mathsf{S})$ is equipped with a Riemannian metric
$g^{(\mathsf{S})}$. Then the following properties are equivalent:
\begin{enumerate} 
\item Let $X: \mathsf{Z} \to \mathsf{X}$
be a coarse graining, and consider a proper parametrisation  $\varphi: \Xi \to {\mathcal P}(\mathsf{Z})$ of  
a cylindrical model ${\mathcal M}$ in ${\mathcal P}(\mathsf{Z})$. 
Assume that $X_\ast \circ \varphi$ is a proper parametrisation of 
the model ${\mathcal M}_X := X_\ast ({\mathcal M})$ in ${\mathcal P}(\mathsf{X})$.
Then for every non-singular point $p \in {\mathcal M}$ satisfying that $X_\ast (p)$ is also a non-singular point of ${\mathcal M}_X$, 
and all smooth functions $f: {\mathcal P}(\mathsf{X}) \to {\Bbb R}$ , we have  
\begin{equation} \label{identitaet2}
      {d X_\ast}_p \left( {\rm grad}^{\mathcal M}_{p} (f \circ X_\ast) \right) \, = \, {\rm grad}^{{\mathcal M}_X}_{X_\ast (p)} f , 
\end{equation}
where the gradient on the LHS is evaluated with respect to the restriction of $g^{(\mathsf{Z})}$ 
and the RHS is evaluated with respect to the restriction of $g^{(\mathsf{X})}$.
\item There exists a positive real number $\alpha$ such that for all $\mathsf{S}$, the metric $g^{(\mathsf{S})}$ coincides with the Fisher-Rao metric multiplied by $\alpha$.
\end{enumerate}
\end{theorem}

\begin{proof} 
``(1) $\Rightarrow$ (2):'' We choose the particular cylindrical sub-manifold ${\mathcal M}(K)$ 
of ${\mathcal P}(\mathsf{Z})$. In this case, (\ref{identitaet2}) is equivalent to $\left. X_\ast \right|_{{\mathcal M}(K)}$ being an isometry between ${\mathcal M}(K)$ and 
${\mathcal P}(\mathsf{X})$. 
On the other hand, according to Chentsov's well-known result \citep{Che82}, this invariance characterises the Fisher-Rao metric up to a constant 
$\alpha > 0$ (see also \citep{AJLS17}).\\
``(2) $\Rightarrow$ (1):'' This follows from the invariance property (\ref{innprod}), which holds for the Fisher-Rao metric, and 
 Theorem \ref{theinvargrad} .  
\end{proof}

\subsection{Cylindrical extensions of a model} 
Throughout this section, we 
consider a model ${\mathcal M}$, together with 
a proper parametrisation ${\Bbb R}^d \supseteq \Xi \to {\mathcal P}(\mathsf{Z})$, 
$\xi \mapsto p_\xi \in {\mathcal M}$, satisfying that the composition $\xi \mapsto X_\ast (p_\xi)$
is a proper parametrisation of ${\mathcal M}_X := X_\ast ({\mathcal M})$. This ensures that all tangent spaces in non-singular points of ${\mathcal M}$ and 
${\mathcal M}_X$, respectively, can be generated in terms 
of partial derivatives with respect to the parameters $\xi_i$, $i = 1, \dots, d$.   
\medskip

We can easily construct a model 
$\widetilde{\mathcal M} \subseteq {\mathcal P}(\mathsf{Z})$ that satisfies the conditions
\begin{equation} \label{cylext}
   {\rm (a)} \;\; {\mathcal M} \subseteq \widetilde{\mathcal M}, \qquad 
   {\rm (b)} \;\; X_\ast({\mathcal M}) = X_\ast(\widetilde{\mathcal M}), \qquad \mbox{and} \qquad {\rm (c)} \;\; \mbox{$\widetilde{\mathcal M}$ is cylindrical}. 
\end{equation}
We refer to such a model as a {\em cylindrical extension of ${\mathcal M}$\/}. Before we come to the explicit construction of cylindrical extensions, let us first 
demonstrate their direct use for relating the respective natural gradients to each other.   
Given a non-singular
point $p \in {\mathcal M}$ that is also non-singular in $\widetilde{\mathcal M}$ and has a non-singular projection $X_\ast(p)$,   
we can decompose the tangent space $T_p \widetilde{\mathcal M}$ into the sum $T_p {\mathcal M} \oplus T^\perp_p {\mathcal M}$, where 
the second summand is the orthogonal complement of the first one in $T_p \widetilde{\mathcal M}$. 
We can use this decomposition in order to relate the natural gradient of a smooth function $f$ defined on the projected model ${\mathcal M}_X$ 
to the natural gradient of $f \circ X_\ast$: 
\begin{eqnarray}
    {\rm grad}^{{\mathcal M}_X}_{X_\ast(p)} f 
       & \stackrel{{d X_\ast}_p}{\leftarrow} &  
                  {\rm grad}_p^{\widetilde{\mathcal M}} (f \circ X_\ast) \nonumber \\ 
       & = &  {\rm grad}_p^{\top} (f \circ X_\ast) +  {\rm grad}_p^{\bot} (f \circ X_\ast) \nonumber \\
       & = & {\rm grad}_p^{\mathcal M} (f \circ X_\ast) +  {\rm grad}_p^{\bot} (f \circ X_\ast). \label{generalis}
\end{eqnarray}
(Here ``$\top$'' stands for the projection onto $T_p {\mathcal M}$ and ``$\bot$'' stands for the projection onto the corresponding orthogonal complement in 
$T_p \widetilde{\mathcal M}$.) The difference between the natural gradient on the full model ${\mathcal M}$ and the natural gradient on the 
coarse grained model ${\mathcal M}_X$ is given by 
${\rm grad}^{\bot} (f \circ X_\ast)$ which vanishes when ${\mathcal M}$ itself is already cylindrical. 
Thus, the equality (\ref{generalis}) generalises (\ref{identitaet2}). 

\paragraph{The product extension I}  
Given a non-singular point $p_\xi = \sum_{z} p(z;\xi) \, \delta^z$ of ${\mathcal M}$, the tangent space in $p_\xi$ is spanned by 
\begin{equation} \label{tangentvec}
    \partial_i(\xi) \, := \, \sum_{z \in \mathsf{Z}} \frac{\partial p(z; \cdot)}{\partial \xi_i} (\xi) \, \delta^z \, = \, 
    \sum_{z \in \mathsf{Z}} p(z; \xi) \frac{\partial \ln p(z; \cdot)}{\partial \xi_i} (\xi) \, \delta^z, \qquad i = 1, \dots,  d.  
\end{equation}
Now, consider the projection of $p_\xi$ onto ${\mathcal P}(\mathsf{X})$ in terms of $X_\ast$, that is 
$X_\ast (p_\xi) = \sum_{x \in \mathsf{X}}  p(x;\xi) \, \delta^x$ where $p(x;\xi) = \sum_{z \in \mathsf{Z}_x} p(z; \xi)$. 
Assuming that this projected point is a non-singular point of ${\mathcal M}_X = X_\ast ({\mathcal M})$, 
the corresponding tangent space $T_{X_\ast (p_\xi)} {\mathcal M}_X$ is spanned by 
\begin{eqnarray} 
    \bar{\partial}^{\mathcal H}_i(\xi) 
    & := & {d X_\ast}_\xi (\partial_i(\xi)) \; = \;  \sum_{x \in \mathsf{X}}  \frac{\partial p(x; \cdot)}{\partial \xi_i}(\xi) \, \delta^x \nonumber \\
     & = & \sum_{x \in \mathsf{X}} p(x; \xi) \frac{\partial \ln p(x; \cdot)}{\partial \xi_i}(\xi) \, \delta^x , \qquad i = 1,\dots, d.   \label{horiz}
\end{eqnarray}
In addition to the described projection of $p_\xi$ onto the ``horizontal'' space, leading to ${\mathcal M}_X$, 
we can also project  it onto the ``vertical'' space. In order to do so, we define a Markov kernel $K_\xi = \sum_{x,z} p(z |x ; \xi) \, \delta^z \otimes e_x$:
\begin{equation} \label{condzx}
     p(z | x; \xi) :=
     \left\{ 
        \begin{array}{c@{,\quad}l}
           \frac{p(z ; \xi)}{p(x ; \xi)} & \mbox{if $X(z) = x$} \\ 
           0                       & \mbox{otherwise} .
        \end{array}
     \right.
\end{equation}
We denote the image of the map $\xi \mapsto K_\xi$ by ${\mathcal M}_{Z|X} \subseteq {\mathcal K}(\mathsf{Z} | \mathsf{X})$, and assume that $K_\xi$ is a 
non-singular point of ${\mathcal M}_{Z|X}$. 
The corresponding tangent vectors in $K_\xi$ are given by 
\begin{eqnarray} 
    \bar{\partial}^{\mathcal V}_i(\xi)
    & := & \sum_{x, z}  \frac{\partial p(z |x ; \cdot)}{\partial \xi_i}(\xi) \, \delta^z \otimes e_x \nonumber \\
    &  = & \sum_{x, z} p(z | x; \xi) \frac{\partial \ln p(z | x; \cdot)}{\partial \xi_i}(\xi) \,   \delta^z \otimes e_x ,  \qquad i = 1,\dots, d.  \label{vert}
\end{eqnarray}
Note that for all three sets of vectors, $\partial_i(\xi)$, $ \bar{\partial}^{\mathcal H}_i(\xi)$, and 
$ \bar{\partial}^{\mathcal V}_i(\xi)$, $i = 1,\dots,d$, linear independence is not required.  
In fact, it is important to include overparametrised systems into the analysis, where linear independence is not given. 
\medskip

Now, we can define the {\em product extension\/} $\widetilde{\mathcal M}^{I}$  of ${\mathcal M}$  
as follows: for each pair $({\xi}, {\xi'}) \in \Xi \times \Xi$, we define $p_{\xi,\xi'} = p(\cdot ; \xi, \xi')$ as 
\begin{eqnarray} 
      \sum_z p(z; \xi, \xi') \, \delta^z 
         & :=  & \sum_{x} \sum_{z \in \mathsf{Z}_x} \left[ p(z ; \xi) + p(x ; \xi) \big(  p(z | x ; \xi') -  p(z | x ; \xi) \big) \right] \, \delta^z \nonumber \\
         &  =  & \sum_{x} \sum_{z \in \mathsf{Z}_x}  p(x ; \xi) \, p(z | x ; \xi') \, \delta^z . \label{firstex}
\end{eqnarray}
The product extension is then simply the set of all points that can be obtained in this way. Obviously, ${\mathcal M}$ consists of those points in 
$\widetilde{\mathcal M}^{I}$ that are given by identical parameters, that is $\xi = \xi'$, which proves (\ref{cylext}) (a).
Furthermore, $X_\ast(p_{\xi,\xi'}) = X_\ast(p_{\xi})$, and therefore this extension has the same projection as the original model ${\mathcal M}$ 
so that (\ref{cylext}) (b) is satisfied. 
The last requirement for $\widetilde{\mathcal M}^{I}$ to be a cylindrical extension of ${\mathcal M}$,  (\ref{cylext}) (c), will be proven below in 
Proposition \ref{horvert}. We obtain the tangent space by taking the derivatives with respect to $\xi_1,\dots,\xi_d$ and $\xi'_1,\dots, \xi'_d$, respectively:
\begin{eqnarray}
    \partial^{\mathcal H}_i (\xi, \xi') 
       & := &  \frac{\partial}{\partial \xi_i} \sum_{z}  p(z; \xi, \xi') \, \delta^z \nonumber \\
       & = &  \sum_{z} p(z; \xi, \xi')  \, \frac{\partial  \ln p(z; \cdot, \xi')}{\partial \xi_i}  (\xi) \, \delta^z \nonumber \\
       & = &  \sum_x \sum_{z \in \mathsf{Z}_x} p(z; \xi, \xi')  \, \frac{\partial  \ln p(x ; \cdot)}{\partial \xi_i}  (\xi) \, \delta^z \nonumber \\
       & = & \sum_{x} \sum_{z \in \mathsf{Z}_x} p(x; \xi) \, p(z | x ; \xi' ) \, \frac{\partial  \ln p(x; \cdot)}{\partial \xi_i}  (\xi) \, \delta^z \nonumber \\
       & = & \sum_{x}  p(x; \xi) \, \frac{\partial  \ln p(x; \cdot)}{\partial \xi_i}  (\xi) \, \left( \sum_{z \in \mathsf{Z}_x} p(z | x ; \xi' ) \, \delta^z \right), \qquad i = 1,\dots,d. 
       \label{hori}
\end{eqnarray}
A comparison with (\ref{horiz}) shows that we have a natural isometric correspondence 
\begin{equation} \label{correspondence1}
    \partial^{\mathcal H}_i (\xi, \xi') \longleftrightarrow \bar{\partial}^{\mathcal H}_i (\xi), \qquad i = 1,\dots, d ,  
\end{equation}
by mapping $\delta^x$ to $\sum_{z \in \mathsf{Z}_x} p(z | x ; \xi' ) \, \delta^z$ (this map is given by the Markov embedding discussed above; see also 
Figure \ref{fig:MarkovEmb}). 
Now we consider the vertical directions:
\begin{eqnarray}
    \partial^{\mathcal V}_i (\xi, \xi') 
       & := & \frac{\partial}{\partial \xi_i '} \sum_{z}  p(z; \xi, \xi') \, \delta^z \nonumber \\  
       & = &  \sum_{z} p(z; \xi, \xi')  \, \frac{\partial  \ln p(z; \xi, \cdot )}{\partial \xi_i'}  (\xi') \, \delta^z \nonumber \\
       & = &  \sum_x \sum_{z \in \mathsf{Z}_x} p(z; \xi, \xi')  \, \frac{\partial  \ln p(z | x ; \cdot)}{\partial \xi_i'}  (\xi') \, \delta^z \nonumber \\
       & = & \sum_{x} \sum_{z \in \mathsf{Z}_x} p(x; \xi) \, p(z | x ; \xi' ) \, \frac{\partial  \ln  p(z | x ; \cdot) }{\partial \xi_i'}  (\xi') \, \delta^z, \qquad i = 1,\dots, d.
       \label{verti}
\end{eqnarray}
A comparison with (\ref{vert}) shows that we also have a natural correspondence 
\begin{equation} \label{correspondence2}
 \partial^{\mathcal V}_i (\xi, \xi') \longleftrightarrow \bar{\partial}^{\mathcal V}_i (\xi'), \qquad i = 1, \dots, d, 
\end{equation}
by mapping $\delta^z \otimes e_x$ to $p(x ; \xi) \, \delta^z$, in addition to the above-mentioned correspondence (\ref{correspondence1}). 
This proves that $(\xi,\xi') \mapsto p_{\xi,\xi'}$ is a proper parametrisation of $\widetilde{\mathcal M}^I$. 
The situation is illustrated in Figure \ref{fig:correspondence}. 

\begin{figure}[h]        
  \centering
       \includegraphics[width=8cm]{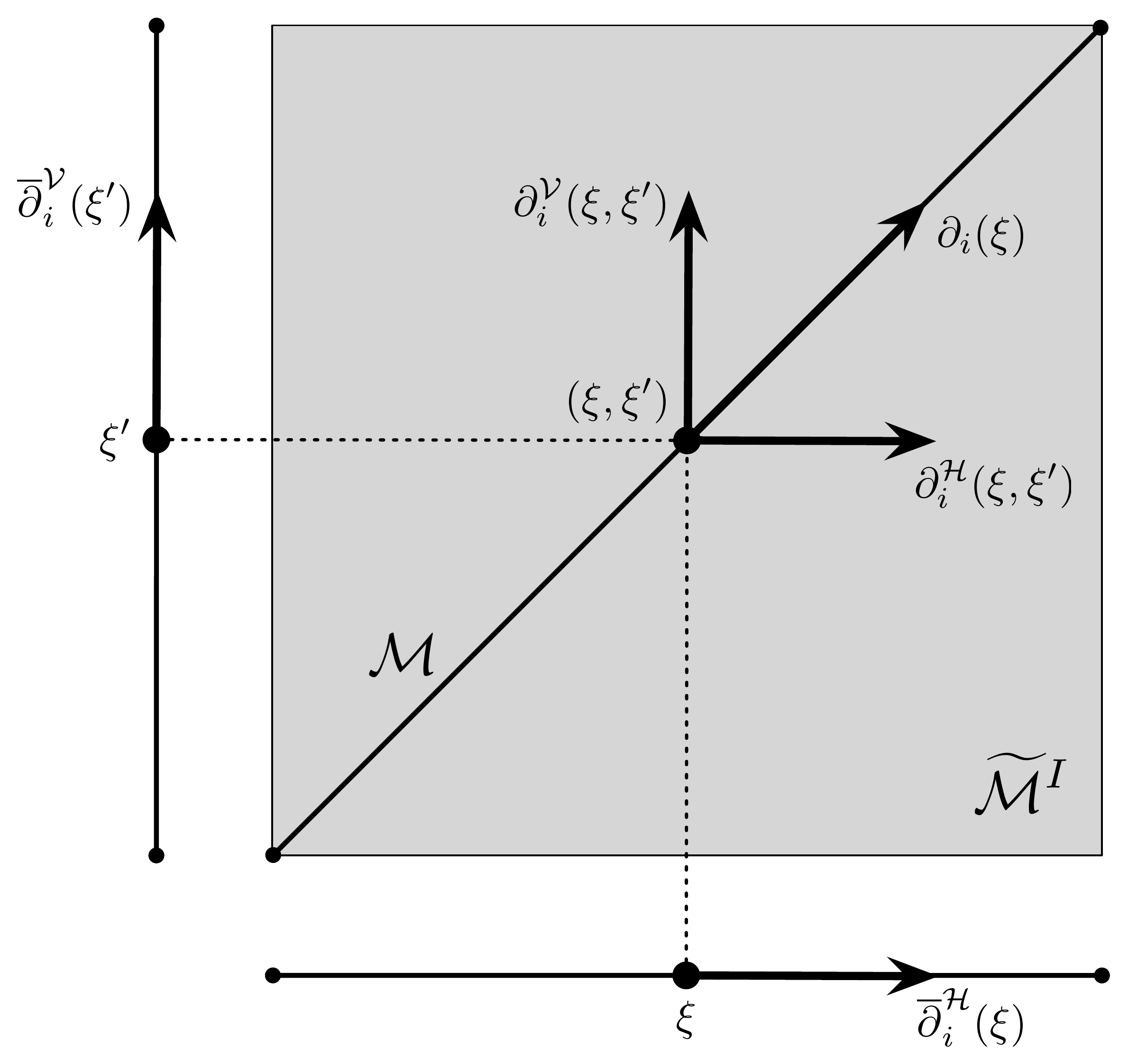}   
  \caption{Extension of ${\mathcal M}$ to the cylindrical model $\widetilde{\mathcal M}^{I}$, with the corresponding tangent vectors.}  
  \label{fig:correspondence} 
\end{figure}
\medskip

Now we consider the natural Fisher-Rao metric on $\widetilde{\mathcal M}^{I} \subseteq {\mathcal P}(\mathsf{Z})$ in $(\xi,\xi')$, assuming that 
all points associated with $(\xi,\xi')$ are non-singular. 
It follows from Proposition \ref{horvert} below that ${\langle \partial_i^{\mathcal H} (\xi, \xi'),  \partial_j^{\mathcal V}(\xi,\xi') \rangle}_{\xi,\xi'} = 0$ for all $i,j$,  
where ${\langle \cdot, \cdot \rangle}_{\xi,\xi'}$ denotes the Fisher-Rao metric.    
For the inner products of the horizontal vectors we obtain
\begin{eqnarray}
  g^{\mathcal H}_{ij} (\xi,\xi') 
  & := & {\left\langle  \partial^{\mathcal H}_i (\xi, \xi'),  \partial^{\mathcal H}_j (\xi, \xi') \right\rangle}_{\xi,\xi'} \nonumber \\
  & =  & \sum_x  p(x; \xi) \sum_{z \in \mathsf{Z}_x} p(z | x ; \xi') \, \frac{\partial \ln p(x; \cdot)}{\partial \xi_i} (\xi) \, \frac{\partial \ln p(x; \cdot)}{\partial \xi_j} (\xi) \nonumber \\
  & = &  \sum_x p(x; \xi) \, \frac{\partial \ln p(x; \cdot)}{\partial \xi_i} (\xi) \, \frac{\partial \ln p(x; \cdot)}{\partial \xi_j} (\xi) \nonumber \\
  & = &  {\left\langle  \overline{\partial}^{\mathcal H}_i (\xi),  \overline{\partial}^{\mathcal H}_j (\xi) \right\rangle}_{\xi} . \label{origfush}
\end{eqnarray}
In particular, these inner products do not depend on $\xi'$. The inner products of the vertical vectors are given by
\begin{eqnarray}
  g^{\mathcal V}_{ij} (\xi,\xi') 
     & := & {\left\langle  \partial^{\mathcal V}_i (\xi, \xi'),  \partial^{\mathcal V}_j (\xi, \xi') \right\rangle}_{\xi,\xi'} \nonumber \\
     & = & \sum_x  p(x; \xi) \sum_{z \in \mathsf{Z}_x} p(z | x ; \xi') \, 
     \frac{\partial \ln p(z | x; \cdot)}{\partial \xi_i'} (\xi') \, \frac{\partial \ln p(z | x; \cdot)}{\partial \xi_j'} (\xi'). 
\end{eqnarray} 
This defines two matrices, 
$G^{\mathcal H}(\xi) = {(g^{\mathcal H}_{ij}(\xi))}_{1 \leq i,j \leq d}$ and 
$G^{\mathcal V}(\xi,\xi') = {(g^{\mathcal V}_{ij}(\xi, \xi'))}_{1 \leq i,j \leq d}$, 
and the Fisher-Rao metric with respect to the product coordinate system
is a block matrix 
\[
    \widetilde{G}(\xi,\xi') = 
    \left( 
       \begin{array}{cc} 
         G^{\mathcal H}(\xi) & 0 \\
         0 & G^{\mathcal V}(\xi,\xi')
       \end{array}
    \right) .
\] 
In order to compute the gradient of a function 
$\widetilde{f} : \widetilde{\mathcal M}^{I} \to {\Bbb R}$, we have to consider the pseudoinverse of 
$\widetilde{G}(\xi,\xi')$, and, with the Euclidean gradient $\nabla_{\xi,\xi'} \widetilde{f} = (\nabla_\xi \widetilde{f}, \nabla_{\xi'} \widetilde{f})$, we have 
\begin{equation} \label{grd}
    {\rm grad}_{\xi,\xi'} \widetilde{f} \; = \; \widetilde{G}^{+}(\xi, \xi') \nabla_{\xi,\xi'} \widetilde{f} = 
    \left( 
       \begin{array}{cc} 
         {G^{\mathcal H}}^{+}(\xi) & 0 \\
         0 & {G^{\mathcal V}}^{+}(\xi,\xi') 
       \end{array}
    \right)
    \left(   
       \begin{array}{c} 
       \nabla_\xi \widetilde{f} \\
       \nabla_{\xi'} \widetilde{f}
       \end{array}
    \right) .
\end{equation}
Now we assume $\widetilde{f} = f \circ X_\ast$, where $f$ is a function defined on the model 
$X_\ast(\widetilde {\mathcal M}) = X_\ast({\mathcal M}) = {\mathcal M}_X$. 
This implies that it only depends on the horizontal variable $\xi$: $\widetilde{f} (p_{\xi,\xi'}) = (f \circ X_\ast) (p_{\xi, \xi'}) = 
f \left( X_\ast (p_{\xi, \xi'})\right) = f(p_\xi)$, and we have  
$\nabla_{\xi'} \widetilde{f} = 0$ and $\nabla_{\xi} \widetilde{f} = \nabla_{\xi} {f}$. With (\ref{grd}), we obtain 
\begin{equation} \label{original}
     {\rm grad}_{\xi,\xi'} (f \circ X_\ast)  \; = \; 
        \left( 
           \begin{array}{c} 
              {G^{\mathcal H}}^{+}(\xi) \nabla_{\xi} f \\
              0      
           \end{array}
        \right) .
\end{equation}
This is a confirmation of our more general result
that the natural gradient on the extended space equals the natural gradient on the reduced space, if the model in the extended space is cylindrical 
(see Theorem \ref{chentsovreform}). However, equation (\ref{original}) does not imply any simplification of the problem, because $G^{\mathcal H}(\xi)$ equals the 
original Fisher information matrix defined for the reduced space ${\mathcal M}_X$ and does not necessarily have a block structure (see equation (\ref{origfush})). 
Assuming that the Fisher information matrix $G(\xi)$ on the full model ${\mathcal M}$ has a block structure, we can try to exploit this structure within its  
product extension $\widetilde{\mathcal M}^I$. For this, note that the tangent vectors (\ref{tangentvec}) of ${\mathcal M}$ in $p_\xi$ can be expressed as 
\begin{equation} \label{tangvecor}
   \partial_i(\xi) \, = \,  \partial^{\mathcal H}_i (\xi) + \partial^{\mathcal V}_i (\xi, \xi) \, \in \, T_\xi {\mathcal M}, \qquad i = 1, \dots, d. 
\end{equation}
This implies $G(\xi) = G^{\mathcal H}(\xi) + G^{\mathcal V} (\xi , \xi)$, and therefore, according to (\ref{original}), we have to invert 
$G^{\mathcal H}(\xi) = G(\xi) - G^{\mathcal V}(\xi,\xi)$, a difference of two matrices where the first one has a block structure and the second one does not. 
This shows 
that the block structure of $G(\xi)$ is not sufficient for the simplification of the problem.  
In what follows, we modify the product extension $\widetilde{\mathcal M}^I$ and open up the possibility for simplification. 
The main idea here parallels the idea of introducing a recognition model, in addition to the generative model, as we did in the context of the wake-sleep 
algorithm in Section \ref{wsa}.

\paragraph{The product extension II} We now generalise the first product extension and replace (\ref{firstex}) by $p_{\xi, \eta} = p(\cdot; \xi, \eta)$ where 
\begin{eqnarray*}
      \sum_z p(z; \xi, \eta) \, \delta^z 
         & :=  & \sum_{x} \sum_{z \in \mathsf{Z}_x} \left[ p(z ; \xi) + p(x ; \xi) \big(  q(z | x ; \eta) -  p(z | x ; \xi) \big) \right] \, \delta^z \\
         &  =  & \sum_{x} \sum_{z \in \mathsf{Z}_x}  p(x ; \xi) \, q(z | x ; \eta) \, \delta^z ,
\end{eqnarray*}
where we denote by $q$ the elements of a model ${\mathcal L}_{Z|X}$ that is properly parametrised by 
$\eta = (\eta_1,\dots,\eta_{d'}) \in {\rm H} \subseteq {\Bbb R}^{d'}$ and contains the model 
${\mathcal M}_{Z|X}$. 
That is, for each $\xi$ there is a $\eta = \eta(\xi)$ such that $p(z | x ; \xi) = q(z | x ; \eta)$. This is closely related to the recognition model 
discussed in Section \ref{wsa}.

Consider a pair $(\xi, \eta) \in \Xi \times {\rm H}$ so that all points associated with it are non-singular points of the respective models. 
For the horizontal and vertical vectors we obtain, analogous to (\ref{hori}) and (\ref{verti}), 
\begin{eqnarray}
     \partial_i^{\mathcal H}(\xi, \eta)  
      & = &  \sum_{x}  p(x; \xi) \, \frac{\partial  \ln p(x; \cdot)}{\partial \xi_i}  (\xi) \, \left( \sum_{z \in \mathsf{Z}_x} q(z | x ; \eta ) \, \delta^z \right), \qquad i = 1,\dots, d,\\
     \partial_i^{\mathcal V}(\xi , \eta) 
     & = & \sum_{x} \sum_{z \in \mathsf{Z}_x} p(x; \xi) \, q(z | x ; \eta ) \, \frac{\partial  \ln  q(z | x ; \cdot) }{\partial \eta_i}  (\eta) \, \delta^z , \qquad i = 1,\dots, d'.   
\end{eqnarray}
Furthermore,
\begin{eqnarray}
    \left\langle  \partial_i^{\mathcal H}(\xi, \eta)  ,  \partial_j^{\mathcal V}(\xi, \eta)  \right\rangle_{\xi, \eta} 
       & = & 0 , \\
    \left\langle  \partial_i^{\mathcal H}(\xi, \eta)  ,  \partial_j^{\mathcal H}(\xi, \eta)  \right\rangle_{\xi, \eta}     
       & = & \sum_{x} p(x ; \xi) \, \frac{\partial \ln p(x ; \cdot)}{\partial \xi_i} (\xi) \, \frac{\partial \ln p(x ; \cdot)}{\partial \xi_j} (\xi) , \\
    \left\langle  \partial_i^{\mathcal V}(\xi, \eta)  ,  \partial_j^{\mathcal V}(\xi, \eta)  \right\rangle_{\xi, \eta}     
       & = & \sum_{x} p(x ; \xi) \sum_{z \in \mathsf{Z}_x} \, \frac{\partial \ln q(z | x ; \cdot)}{\partial \eta_i} (\eta) \, \frac{\partial \ln q(z | x ; \cdot)}{\partial \eta_j } (\eta) .    
\end{eqnarray}
\begin{figure}[h]        
  \centering
       \includegraphics[width=8cm]{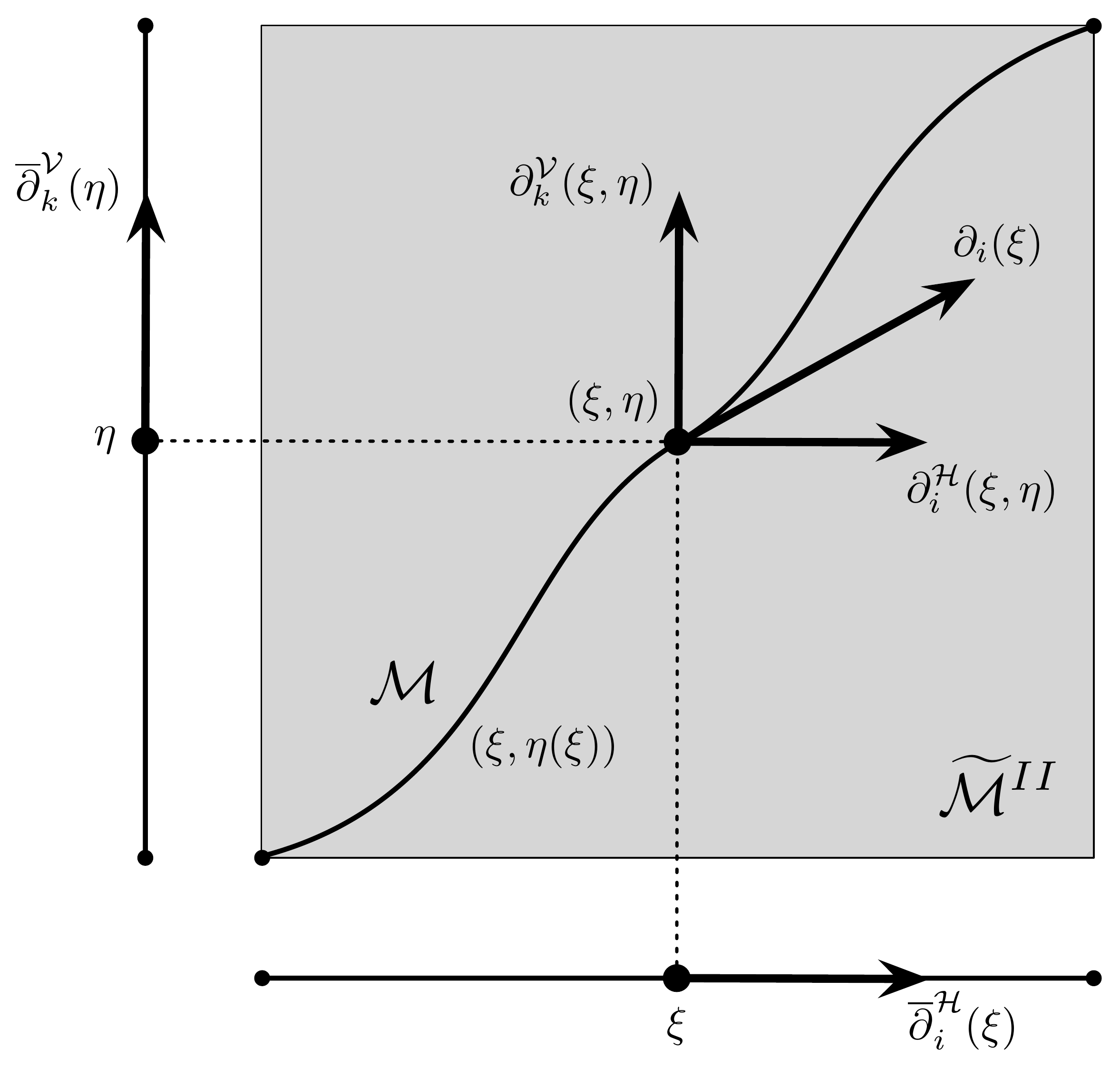}   
  \caption{Extension of ${\mathcal M}$ to the cylindrical model $\widetilde{\mathcal M}^{II}$, with the corresponding tangent vectors.}  
  \label{fig:correspondence2} 
\end{figure}
For the gradient of a function $f$ on ${\mathcal M}_X$, we obtain the same formula as (\ref{original}). However, with the second product extension we can choose the 
model ${\mathcal L}_{Z|X}$ to be larger than ${\mathcal M}_{Z|X}$. This provides a way to simplify $G^{\mathcal H}(\xi)$ in (\ref{original}). 
In order to be more explicit, consider the natural embedding of ${\mathcal M}$ into $\widetilde{\mathcal M}^{II}$, 
\begin{equation}
   \xi \; \; \mapsto \; \; \sum_z p(z ; \xi, \eta (\xi)) \, \delta^z . 
\end{equation}     
For the tangent vectors we now obtain
\begin{eqnarray*}
   \partial_i(\xi) 
      & = & \sum_z p(z ; \xi, \eta(\xi)) \, \frac{\partial \ln p(z ; \cdot , \eta(\cdot))}{\partial \xi_i} (\xi) \, \delta^z \\
      & = & \sum_x \sum_{z \in \mathsf{Z}_x} 
                p(x ; \xi) \, q(z | x ; \eta(\xi)) \, \left[ \frac{\partial \ln p(x ; \cdot )}{\partial \xi_i} (\xi) + 
                \frac{\partial \ln q(z | x ; \eta (\cdot))}{\partial \xi_i} (\xi) \right] \, \delta^z \\
   & = & \sum_x \sum_{z \in \mathsf{Z}_x} 
                p(x ; \xi) \, q(z |x ; \eta (\xi)) \, \frac{\partial \ln p(x ; \cdot )}{\partial \xi_i} (\xi)  \, \delta^z  \\
   & &             +    
               \sum_x \sum_{z \in \mathsf{Z}_x} p(x ; \xi) \, 
               q(z |x ; \eta(\xi)) \sum_k \frac{\partial \ln q(z | x ; \cdot )}{\partial \eta_k} (\eta (\xi)) \, \frac{\partial \eta_k}{\partial \xi_i} (\xi) \, \delta^z \\
   & &  \mbox{(by the chain rule)} \nonumber \\          
      & = & \partial_i^{\mathcal H}(\xi , \eta(\xi)) + \sum_{k} \frac{\partial \eta_k}{\partial \xi_i} (\xi) \, \partial_k^{\mathcal V}(\xi , \eta (\xi)) . 
\end{eqnarray*}
This derivation generalises the equation (\ref{tangvecor}). For the Fisher information matrix $G(\xi) = (g_{ij}(\xi))_{1\leq i, j \leq d}$ we obtain
\begin{eqnarray*}
    g_{ij}(\xi) & = &  \left\langle \partial_i (\xi) , \partial_j(\xi) \right\rangle_{\xi} \\
        & = & \left\langle \partial_i^{\mathcal H}(\xi , \eta (\xi)) , \partial_j^{\mathcal H}(\xi , \eta (\xi)) \right\rangle_{\xi} + 
                 \sum_{k,l} \frac{\partial \eta_k}{\partial \xi_i} (\xi) \, \frac{\partial \eta_l}{\partial \xi_j} (\xi) 
                 \left\langle \partial_k^{\mathcal V}(\xi , \eta(\xi)) , \partial_l^{\mathcal V}(\xi , \eta (\xi)) \right\rangle_{\xi} \\
       & = & g^{\mathcal H}_{ij} (\xi) + \sum_{k,l} \frac{\partial \eta_k}{\partial \xi_i} (\xi) \, \frac{\partial \eta_l}{\partial \xi_j} (\xi) \, g^{\mathcal V}_{kl}(\xi, \eta (\xi)) .          
\end{eqnarray*}
Thus, we can insert  
\begin{equation} \label{diffcompl}
    g^{\mathcal H}_{ij} (\xi)  \; = \;  g_{ij}(\xi) - \sum_{k,l} \frac{\partial \eta_k}{\partial \xi_i} (\xi) \, \frac{\partial \eta_l}{\partial \xi_j} (\xi) \, g^{\mathcal V}_{kl}(\xi, \eta(\xi)).  
\end{equation}
into equation (\ref{original}). At first sight, this does not appear to simplify the problem. However, as we will outline in the next section, 
it suggests conditions for both, the generative model as well as the recognition model, that would be sufficient for a simplification of 
$G^{\mathcal H}(\xi)$. These conditions involve locality properties, as we studied in Section \ref{bnn}, but also an appropriate coupling between the two models.  
\medskip

We now prove that the second product extension, and thereby also the first one, are indeed cylindrical extensions of ${\mathcal M}$.  

\begin{proposition} \label{horvert}
The product extensions $\widetilde{\mathcal M}^{II}$ and, as a special case, $\widetilde{\mathcal M}^I$ are cylindrical extensions of ${\mathcal M}$. 
More precisely, we have 
\begin{eqnarray}
     T_{\xi,\eta} \widetilde{\mathcal M}^{II} \cap {\mathcal H}_{\xi,\eta} 
     & = &  {\rm span} \left\{ \partial^{\mathcal H}_i(\xi,\eta) \, : \, i = 1,\dots, d \right\} \label{eq1} \\
      T_{\xi,\eta} \widetilde{\mathcal M}^{II} \cap {\mathcal V}_{\xi,\eta} 
     & = &  {\rm span} \left\{ \partial^{\mathcal V}_i(\xi,\eta) \, : \, i = 1,\dots, d' \right\} \label{eq2} \\
     T_{\xi,\eta} \widetilde{\mathcal M}^{II} 
     & = & \left( T_{\xi,\eta} \widetilde{\mathcal M}^{II} \cap {\mathcal H}_{\xi,\eta} \right) +   
               \left( T_{\xi,\eta} \widetilde{\mathcal M}^{II} \cap {\mathcal V}_{\xi,\eta} \right) . \label{eq3}
\end{eqnarray}
\end{proposition}

\begin{proof} We have to verify the properties (a), (b), and (c) in (\ref{cylext}). \\
(a) We have assumed that for each $\xi$ there is a $\eta = \eta(\xi)$ such that $p(z | x ; \xi) = q(z | x ; \eta(\xi))$. This implies that each distribution $p_\xi \in {\mathcal M}$
is also contained in $\widetilde{\mathcal M}^{II}$:
\[
 p(z ; \xi) \, = \, p(x ; \xi) p(z | x ; \xi) \, = \,  p(x ; \xi) q(z | x ; \eta(\xi)) \, = \, p(z ; \xi, \eta(\xi)).
\]
(b) Clearly, from (a) we obtain $X_\ast ({\mathcal M}) \subseteq X_\ast (\widetilde{\mathcal M}^{II})$. To prove the opposite inclusion, we consider a point
$p_{\xi,\eta} \in \widetilde{\mathcal M}^{II}$ and show that the point $p_\xi \in {\mathcal M}$ has the same $X_\ast$-projection:
\begin{eqnarray*}
     X_\ast \left( p_{\xi, \eta} \right) 
        & = & X_\ast \left( \sum_{x} \sum_{z \in \mathsf{Z}_x} p(x ; \xi) \, q(z| x ; \eta) \, \delta^z \right) \\
        & = & \sum_{x} \left( \sum_{z \in \mathsf{Z}_x} p(x ; \xi) \, q(z| x ; \eta) \right) \, \delta^x \\
        & = & \sum_{x} p(x ; \xi) \, \delta^x  \\
        & = & \sum_{x} \left( \sum_{z \in \mathsf{Z}_x} p(x ; \xi) \, q(z | x ; \eta (\xi)) \right) \, \delta^x \\
        & = & \sum_{x} \left( \sum_{z \in \mathsf{Z}_x} p(x ; \xi) \, p(z | x ; \xi) \right) \, \delta^x \\
        & = & \sum_{x} \left( \sum_{z \in \mathsf{Z}_x} p(z ; \xi) \right) \, \delta^x \; = \; X_\ast \left( \sum_z p(z ; \xi) \, \delta^z \right) \; = \; X_\ast \left( p_\xi \right). 
\end{eqnarray*}
(c) We have 
\[
    {\mathcal H}_{\xi, \eta} \, := \, 
    \left\{  \widetilde{U} = \sum_{x} U(x) \sum_{z \in \mathsf{Z}_x} q(z | x ; \eta) \, \delta^z   \; : \; \sum_x U(x) = 0 \right\} 
\]
with the orthogonal complement 
\[
     {\mathcal V}_{\xi,\eta} \, := \, 
     \left\{ \sum_{z} V(z) \, \delta^{z}  \; : \; \mbox{$\sum_{z \in \mathsf{Z}_x} V(z) \, = \, 0$ for all $x$} \right\} .
\] 
We first show that the horizontal vectors 
\begin{equation*}
 \partial^{\mathcal H}_i (\xi, \eta)   
    \; = \; \sum_{x} p(x; \xi) \,  \frac{\partial  \ln p(x; \cdot)}{\partial \xi_i}  (\xi) \sum_{z \in \mathsf{Z}_x}  q(z | x ; \eta ) \, \delta^z 
\end{equation*}
are contained in ${\mathcal H}_{\xi, \eta}$. To this end, 
we set $U(x) = p(x; \xi) \,  \frac{\partial  \ln p(x; \cdot)}{\partial \xi_i}  (\xi)$ and verify 
\begin{eqnarray*}
   \sum_x U(x) 
      & = & \sum_x p(x; \xi) \,  \frac{\partial  \ln p(x; \cdot)}{\partial \xi_i}  (\xi) \\
      & = & \sum_x \left.\frac{\partial  p(x; \cdot)}{\partial \xi_i}  (\xi) \; = \; \frac{\partial}{\partial \xi_i} \sum_x p(x; \cdot) \right|_{\xi} \\
      & = & 0. 
\end{eqnarray*}
Now we show that the vertical vectors
\begin{equation*}    
 \partial^{\mathcal V}_i (\xi, \eta)  
     \; = \; \sum_{x} \sum_{z \in \mathsf{Z}_x} p(x; \xi) \, q(z | x ; \eta ) \, \frac{\partial  \ln  q(z | x ; \cdot) }{\partial \eta_i}  (\eta) \, \delta^z 
\end{equation*}
are contained in ${\mathcal V}_{\xi,\eta}$. 
We set $V(z) := p(X(z); \xi) \, q(z | X(z) ; \eta ) \, \frac{\partial  \ln  q(z | X(z) ; \cdot) }{\partial \eta_i}  (\eta)$ and verify 
\begin{eqnarray*} 
\sum_z V(z) 
   & = & \sum_{x} \sum_{z \in \mathsf{Z}_x} p(x; \xi) \, q(z | x ; \eta ) \, \frac{\partial  \ln  q(z | x ; \cdot) }{\partial \eta_i}  (\eta) \\
   & = & \sum_{x} \sum_{z \in \mathsf{Z}_x} p(x; \xi) \, \frac{\partial   q(z | x ; \cdot) }{\partial \eta_i}  (\eta) \\
   & = &  \sum_{x} p(x; \xi) \left. \frac{\partial}{\partial \eta_i} \sum_{z \in \mathsf{Z}_x} q(z | x ; \cdot) \right|_{\eta} \\
   & = & 0. 
\end{eqnarray*}
In conclusion, we have 
\begin{eqnarray*}
  T_{\xi, \eta} \widetilde{\mathcal M}^{II} 
     & = & {\rm span} \left\{ \partial^{\mathcal H}_i(\xi, \eta) \, : \, i = 1,\dots, d \right\} + 
               {\rm span} \left\{ \partial^{\mathcal V}_i(\xi,\eta) \, : \, i = 1,\dots, d' \right\} \\
     & \subseteq &  \left( T_{\xi,\eta} \widetilde{\mathcal M}^{II} \cap {\mathcal H}_{\xi,\eta} \right) + 
                             \left( T_{\xi,\eta} \widetilde{\mathcal M}^{II} \cap {\mathcal V}_{\xi,\eta} \right) \\
     & \subseteq & T_{\xi, \eta} \widetilde{\mathcal M}^{II} ,                        
\end{eqnarray*}
which proves the equalities (\ref{eq1}), (\ref{eq2}), and (\ref{eq3}). 
\end{proof}

\section{Conclusions: A natural gradient perspective of the wake-sleep algorithm}  \label{conclu}
Information geometry provides two natural geometries associated with a learning system that has visible units $V$ and hidden units $H$. 
Typically, the system is given in terms of a model ${\mathcal M}$ of probability distributions of global states of the full system, 
${\mathcal P}_{V,H}$, but the objective function $f$ only depends on the probability distribution of the visible states, giving rise to 
a projected model ${\mathcal M}_V \subseteq {\mathcal P}_V$. Both geometric objects, ${\mathcal M}$ and ${\mathcal M}_V$, 
carry a natural geometry inherited from the respective ambient space. In Section \ref{bnn} we studied various locality properties 
of the natural gradient based on the first geometry, thereby assuming a factorisation of the elements of ${\mathcal M}$ according to a 
directed acyclic graph. These properties simplify the Fisher information matrix for ${\mathcal M}$ 
and allow us to apply the natural gradient method to deep networks. The second geometry, 
the geometry of ${\mathcal M}_V$, was studied in Section \ref{singularnatgrad} where we took a somewhat more general perspective.  
In what follows, we restate the general problem of comparing the two mentioned geometries 
within that perspective and summarise the corresponding results.   
\medskip

Consider a model ${\mathcal S}$ in the set ${\mathcal P}(\mathsf{X})$ of probability distributions on a finite 
set $\mathsf{X}$, that is ${\mathcal P}(\mathsf{X})$, and a smooth 
function $f: {\mathcal P}(\mathsf{X}) \to {\Bbb R}$. 
The task is to optimise $f$ on ${\mathcal S}$ in terms of the natural gradient ${\rm grad}^{\mathcal S} f$. 
With no further assumptions this can be a very difficult problem. Typically, however, ${\mathcal S}$ is obtained as 
the image of a simpler model ${\mathcal M}$ of probability distributions on a larger set $\mathsf{Z}$, 
${\mathcal P}(\mathsf{Z})$. 
More precisely, we consider a surjective map $X: \mathsf{Z} \to \mathsf{X}$, and the corresponding push-forward map 
$X_\ast : {\mathcal P}(\mathsf{Z}) \to {\mathcal P}(\mathsf{X})$ of probability measures. The model ${\mathcal M}_X$ is then nothing but the $X_\ast$-image 
of ${\mathcal M}$, that is ${\mathcal S} = {\mathcal M}_X = X_\ast({\mathcal M})$. Now, instead of optimising $f$ on ${\mathcal M}_X$, we can optimise 
$f \circ X_\ast$ on ${\mathcal M}$ and aim to simplify the problem by exploiting the structure of ${\mathcal M}$. This works to some extent. Even though the  
two problems are closely related,  
the corresponding gradient fields $dX_\ast \left( {\rm grad}^{\mathcal M} (f \circ X_\ast) \right)$ and ${\rm grad}^{{\mathcal M}_X} f$ 
typically differ from each other. Thus, the optimisation of $f$ on ${\mathcal M}_X$ based on the Fisher-Rao metric on 
${\mathcal P}(\mathsf{X})$, and the optimisation of $f \circ X_\ast$ on ${\mathcal M}$ based on the Fisher-Rao metric on ${\mathcal P}(\mathsf{Z})$ are 
not equivalent. We can try to improve the situation by replacing the Fisher-Rao metric on ${\mathcal M}$ and ${\mathcal M}_X$, respectively, 
by different Riemannian metrics. While this might be a reasonable approach for the simplification of the problem, 
from the information-geometric perspective, the Fisher-Rao metric 
is the most natural one, which is the reason for referring to the Fisher-Rao gradient as the 
{\em natural gradient\/}. This is directly linked to the invariance of gradients, as we have highlighted in this article. 
If we request invariance of the gradients for all coarse grainings $X: \mathsf{Z} \to \mathsf{X}$, 
all models ${\mathcal M} \subseteq {\mathcal P}(\mathsf{Z})$ from a particular class, 
and all functions $f: {\mathcal M}_X \to {\Bbb R}$, by Chentsov's classical characterisation theorem, we {\em have\/} to impose the Fisher-Rao metric on the individual models 
(see Theorem \ref{chentsovreform}). Even then, the invariance of gradients is satisfied only if the model 
is cylindrical in the sense of Definition \ref{cylindrical}. Given a model ${\mathcal M}$ that is not cylindrical, we have proposed cylindrical extensions 
$\widetilde{\mathcal M}$ which contain ${\mathcal M}$. The natural gradient of $f$ on ${\mathcal M}_X$
is then equivalent to the natural gradient of $f \circ X_\ast$ on such an extension $\widetilde{\mathcal M}$. \\

As an outlook, we want to touch upon the following two related problems:
\begin{enumerate}
\item Can we exploit the simplicity 
of the original model ${\mathcal M}$ in order to simplify the optimisation on $\widetilde{\mathcal M}$?
\item The original model ${\mathcal M}$ is associated with some network. 
What kind of network can we associate with the extended model $\widetilde{\mathcal M}$?
\end{enumerate}  
We want to briefly address these problems within the context of Section \ref{bnn}, where $\mathsf{X} = \mathsf{X}_{V}$, 
$\mathsf{Z} = \mathsf{X}_{V} \times \mathsf{X}_H$, and $X = X_V: (v, h) \mapsto v$.   
As the cylindrical extension $\widetilde{\mathcal M}^{II}$ suggests,  
it can be associated with the addition of a recognition model ${\mathcal L}_{H|V}$, assuming that ${\mathcal M}$ is a generative model. 
If both models are parametrised by (\ref{parametris}) and (\ref{facrec}), respectively, then the corresponding Fisher information matrices simplify as 
stated in Theorem \ref{localform}. They both have a block structure where each block corresponds to one unit. Outside of these blocks, the matrices are filled with zeros. 
Being more precise, we consider all parameters that correspond to unit $r$, the parameters $\xi_r = (\xi_{(r;1)}, \dots, \xi_{(r;d_r)})$ of the generative model ${\mathcal M}$, 
and the parameters $\eta_r = (\eta_{(r; 1)} , \dots , \eta_{(r; d_r')})$ of the recognition model ${\mathcal L}_{H|V}$. With (\ref{diffcompl}) we then obtain  
\begin{equation} \label{diffcompl2}
    g^{\mathcal H}_{(r;i) (s; j)} (\xi)  \; = \;  
    g_{(r;i) (s; j)}(\xi) - \sum_{t,u} \sum_{(t;k),(u;l)} \frac{\partial \eta_{(t;k)}}{\partial \xi_{(r; i)}} (\xi) \, \frac{\partial \eta_{(u;l)}}{\partial \xi_{(s; j)}} (\xi) \, 
    g^{\mathcal V}_{(t; k)(u; l)}(\xi, \eta(\xi)). 
\end{equation}
We know that $g_{(r;i) (s; j)}(\xi) = 0$ if $r \not= s$ and $g^{\mathcal V}_{(t; k)(u; l)}(\xi, \eta(\xi)) = 0$ if $t \not= u$. With the latter property, the 
sum on the RHS of (\ref{diffcompl2}) reduces to 
\begin{equation} \label{diffcompl3}
   \sum_t \sum_{(t;k),(t;l)} \frac{\partial \eta_{(t;k)}}{\partial \xi_{(r; i)}} (\xi) \, \frac{\partial \eta_{(t;l)}}{\partial \xi_{(s; j)}} (\xi) \, g^{\mathcal V}_{(t; k)(t; l)}(\xi, \eta(\xi)). 
\end{equation}
If all partial derivatives $\partial \eta_{(t;k)} / \partial \xi_{(r; i)}(\xi)$ are local in the sense that they vanish whenever $t \not= r$, 
then the matrix $G^{\mathcal H}(\xi)$ inherits the block structure of the matrices $G(\xi)$ and $G^{\mathcal V}(\xi, \eta(\xi))$. 
However, this is typically not the case and represents an additional coupling between the generative model and the recognition model. 
Without that coupling, the partial derivatives in (\ref{diffcompl3}) will ``overwrite'' the block structure of the matrix $G(\xi)$, leading to a 
non-local matrix $G^{\mathcal H}(\xi)$ with $g^{\mathcal H}_{(r;i) (s; j)} (\xi) \not= 0$ even if $r \not= s$.       
The degree of non-locality will depend on the specific properties of the partial derivatives $\partial \eta_{(t;k)} / \partial \xi_{(r; i)} (\xi)$. 
\medskip

We conclude this article by revisiting the wake-sleep algorithm of Section \ref{wsa}.  
Let us assume that (\ref{diffcompl2}) and (\ref{diffcompl3}) imply a sufficient simplification so that a 
natural gradient step in $\widetilde{\mathcal M}^{II}$ can be made. 
This will update the generation parameters, say from $\xi$ to $\xi + \Delta \xi$, and leave the 
recognition parameters $\eta$ unchanged. Such a an update corresponds to a natural gradient version of the wake step. 
The resulting point $(\xi + \Delta \xi, \eta)$ in $\widetilde{\mathcal M}^{II}$ will typically be outside of ${\mathcal M}$. 
As the simplification through (\ref{diffcompl2}) and (\ref{diffcompl3}) only holds on ${\mathcal M}$, we have to update the recognition parameters, say from 
$\eta$ to $\eta + \Delta \eta$, so that the resulting point $(\xi + \Delta \xi , \eta + \Delta \eta)$ is again in ${\mathcal M}$. This sleep step 
will ensure that the next update of the generation parameters benefits from the simplicity of the Fisher information matrix. The situation is illustrated in Figure \ref{fig:wakesleepillu}.  
\begin{figure}[h]        
  \centering
       \includegraphics[width=8cm]{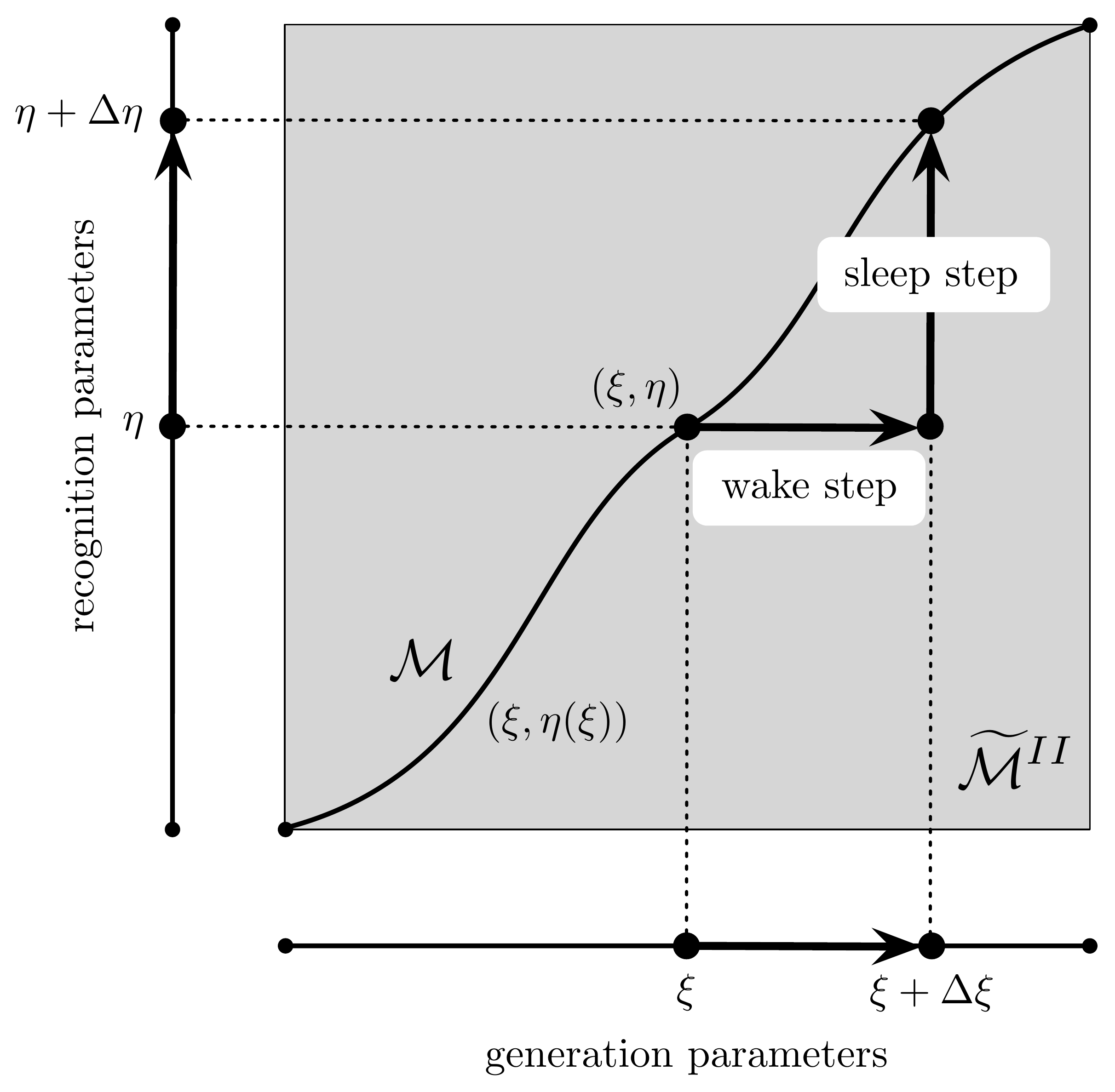}   
  \caption{Illustration of the generalised wake-sleep algorithm, 
  taking place on the cylindrical extension $\widetilde{\mathcal M}^{II}$ of ${\mathcal M}$.}  
  \label{fig:wakesleepillu} 
\end{figure}
Note that it is irrelevant how we get back to ${\mathcal M}$ within the 
sleep step, as far as we do not change the generation parameters. Also, it might be required to 
apply several sleep steps until we get back to ${\mathcal M}$, which highlights the asymmetry of time scales of the two phases. 
This asymmetric version has been outlined and discussed in the context of the $em$-algorithm 
by \cite{IAN98}.
The overall wake-sleep step will typically not follow the gradient of an objective function on $\widetilde{\mathcal M}^{II}$. However, this is not the aim here. The prime process is the process in $\xi$ which parametrises 
${\mathcal M}_V$. Effectively, the outlined version of the wake-sleep algorithm will follow the natural gradient of the objective function with respect to the geometry of 
${\mathcal M}_V$. The natural wake-sleep algorithm with respect to the geometry of ${\mathcal M}$ has been recently studies by \cite{VVMA20}. 
\medskip  

In Section \ref{wsa} we introduced the recognition model as an auxiliary model for sampling, which was required for the evaluation of the gradient with respect 
to $\xi$. This work reveals another role of the recognition model in the context of the natural gradient method. 
It allows us to define an extension of the original model ${\mathcal M}$ 
so that we can effectively apply the natural gradient method on ${\mathcal M}_V$ within the context of deep learning.  
The presented results suggest criteria for the coupling between the generative model and recognition model that would ensure the locality 
of the natural gradient on this projected model.      

\section*{Acknowledgement}
The author is grateful for valuable discussions with Luigi Malag{\`o}, Riccardo Volpi, and Csongor-Huba V{\'a}rady.

\section{Appendix: Moore-Penrose inverse and gradients}
We consider a parametrised model ${\mathcal M}$ with a parametrisation 
$\xi: {\Bbb R}^d \supseteq U \to V \subseteq {\mathcal M}$, $\xi \mapsto p_\xi$. 
For a non-singular point $p_\xi \in {\mathcal M}$, we assume that the tangent space in $p_\xi$, $T_\xi {\mathcal M}$, is spanned by the vectors $\partial_i := \frac{\partial}{\partial \xi_i}$, 
$i = 1,\dots, d$. Note that we do not assume that these vectors are independent. Now consider a function $f: {\mathcal M} \to {\Bbb R}$ that is smooth in $p_\xi$, and its differential 
\[
     {df}_{\xi} : \; T_\xi {\mathcal M} \; \to \; {\Bbb R}, \qquad X \; \mapsto \; {df}_\xi(X) = \frac{\partial f}{\partial X} (\xi). 
\]   
This is a linear form on $T_\xi {\mathcal M}$. With a non-degenerate bilinear form $g_\xi$ on $T_\xi {\mathcal M}$ we can identify ${df}_\xi$ with a vector 
${\rm grad}_\xi f \in T_\xi {\mathcal M}$, which points in the direction of maximal infinitesimal increase of $f$ in $\xi$. It is uniquely characterised by the equation 
\begin{equation} \label{bedi}
                         g_\xi ({\rm grad}_\xi f , X) \, = \, {df}_\xi(X), \qquad X \in T_\xi {\mathcal M} .
\end{equation}
Now, we express the gradient in local coordinates. First, it has a representation 
\begin{equation} \label{darste}
                        {\rm grad}_\xi f \, = \, \sum_{i = 1}^d   x^i \, \partial_i .
\end{equation}
Note that this representation of the gradient in terms of the coefficients $x = (x^1,\dots,x^d)$ is not necessarily unique (due to the fact that the vectors $\partial_i$, $i = 1,\dots,d$, need not be independent). 
We insert the RHS of (\ref{darste}) and $X = \partial_j$ into (\ref{bedi}) and obtain
\begin{equation}
     \sum_{i = 1}^d x^i  \, g_{ij}(\xi) \; = \; \frac{\partial f}{\partial \xi_j} (\xi), \qquad j = 1,\dots, d ,
\end{equation}
or, in matrix notation, 
\begin{equation} \label{gnabla}
        G(\xi) \, x \; = \; \nabla_\xi f .
\end{equation}
Any coefficient vector $x \in {\Bbb R}^d$ will provide an equally valid representation of the gradient in terms of the tangent vectors $\partial_i$. 
Furthermore, we know that there is 
at least one solution $x$ that represents the gradient. In the case where $G(\xi)$ is of maximal rank this solution is unique and 
we can simply apply the inverse of $G(\xi)$ in order to obtain the coefficients of the gradient as $x = G^{-1}(\xi) \nabla_\xi f$. 
This is the usual case when we have a local (diffeomorphic) coordinate system around the point $p_\xi$. Even though we interpret a parametrisation of a model as a 
coordinate system, the number of parameters often exceeds the dimension of the model. In these cases, the matrix $G(\xi)$ will not be of maximal rank 
so that we have a non-trivial kernel ${\rm ker} \, G(\xi)$. We can always add to a solution $x$ of (\ref{gnabla}) a vector $y$ from that kernel and obtain another solution $x + y$.    
The affine space $A = x + {\rm ker} \, G(\xi) \subseteq {\Bbb R}^d$ of solutions describes all possible representations of the gradient in terms of 
$\partial_1,\dots, \partial_d$. They are all equally adequate for describing a learning process that takes place in ${\mathcal M}$. 
However, from the perspective of linear algebra there is a natural choice, the element in the affine 
solution space $A$ that is orthogonal to ${\rm ker} \, G(\xi)$ (with respect to the canonical inner product in ${\Bbb R}^d$). This defines 
the Moore-Penrose inverse $G^+(\xi)$, also called pseudoinverse, which has been previously proposed by several authors (see, e.g., \cite{Tho14}). 
In this paper, we were concerned with a number of simplifications of the natural gradient. 
One simplification was expressed in terms of a block diagonal structure of the Fisher information matrix. For the representation of 
the natural gradient, we evaluated the pseudoinverse of that block diagonal matrix based on the following simple observation 
(see, e.g., \cite{CGMSR15} for more general results related to the pseudoinverse of a block matrix):
\begin{equation}
    {\left(
       \begin{array}{ccc}
          G_{1} &            & 0 \\
                    & \ddots &  \\
           0       &            & G_{N}
       \end{array}
    \right)}^+ 
\; = \; 
    \left(
       \begin{array}{ccc}
           G_{1}^+ &.          & 0 \\
                         & \ddots &    \\
                      0 &        & G_{N}^+
       \end{array}
    \right). 
\end{equation}
\medskip

How natural is the Moore-Penrose inverse? There are two perspectives here. 
On the one hand, $G^+(\xi) \, \nabla_\xi f$ {\em is\/} natural in the sense that it represents an object, ${\rm grad}_\xi f$, that is independent of the parametrisation.   
On the other hand, the inner product used for the definition of $G^+(\xi)$ is the canonical inner product 
in ${\Bbb R}^d$ which does not have to be at all related to the 
metric $g_\xi$. 
In this article, we have chosen the Moore-Penrose inverse as one possible extension of the usual inverse to overparametrised models which has been previoulsy proposed 
by several authors (see, e.g., \cite{Tho14}). However, as outlined in this section, there are also other possibilities for such an extension. We have some  
flexibility here which might allow us to further simplify the representation of the natural gradient in terms of a particular choice of the parametrsiation. 

\vskip 0.2in
\bibliography{locality} \bibliographystyle{plainnat}
\end{document}